\documentclass{article}

\PassOptionsToPackage{numbers,compress}{natbib}
\usepackage[table]{xcolor}
\usepackage[main, final]{neurips_2026}

\usepackage[utf8]{inputenc} 
\usepackage[T1]{fontenc}    
\usepackage{hyperref}       
\usepackage{url}            
\usepackage{booktabs}       
\usepackage{amsfonts}       
\usepackage{nicefrac}       
\usepackage{microtype}      
\usepackage{xcolor}         

\usepackage{amsmath}
\usepackage{color}
\usepackage{multirow}
\usepackage{enumitem}
\usepackage{graphicx}
\usepackage{tcolorbox}

\usepackage{algorithm}
\usepackage{algorithmic}

\usepackage{wrapfig}
\usepackage{fontawesome5}

\definecolor{myred}{HTML}{E74C3C}
\definecolor{myblue}{HTML}{3498DB}
\definecolor{mygreen}{HTML}{27AE60}
\definecolor{myorange}{HTML}{ff8800}
\definecolor{mypurple}{HTML}{9B59B6}
\definecolor{mypink}{HTML}{f7acb9}

\definecolor{headercolor}{RGB}{100,100,100}
\definecolor{bluecolor}{RGB}{0,100,200}
\definecolor{orangecolor}{RGB}{200,100,50}
\definecolor{redcolor}{RGB}{200,0,0}

\newcommand{\red}[1]{\mytextbox{myred}{\textbf{\textcolor{myred}{#1}}}}
\newcommand{\blue}[1]{\mytextbox{myblue}{\textbf{\textcolor{myblue}{#1}}}}
\newcommand{\green}[1]{\mytextbox{mygreen}{\textbf{\textcolor{mygreen}{#1}}}}

\newcommand{\purple}[1]{\mytextbox{mypurple}{\textbf{\textcolor{mypurple}{#1}}}}

\usepackage{tcolorbox}
\newcommand{\mytextbox}[2]{%
  \colorbox{#1!15}{\strut #2}%
}

\newcounter{assumption}
\newenvironment{assumption}[1][]{\refstepcounter{assumption}\par\medskip\noindent\textbf{Assumption~\theassumption\if\relax\detokenize{#1}\relax\else\ (#1)\fi.}\itshape\quad}{\par\medskip}

\newcounter{proposition}
\newenvironment{proposition}[1][]{\refstepcounter{proposition}\par\medskip\noindent\textbf{Proposition~\theproposition\if\relax\detokenize{#1}\relax\else\ (#1)\fi.}\itshape\quad}{\par\medskip}

\newcounter{corollary}
\newenvironment{corollary}[1][]{\refstepcounter{corollary}\par\medskip\noindent\textbf{Corollary~\thecorollary\if\relax\detokenize{#1}\relax\else\ (#1)\fi.}\itshape\quad}{\par\medskip}

\newenvironment{proof}{\par\medskip\noindent\textit{Proof.}\quad}{\hfill$\square$\par\medskip}

\title{A\textsuperscript{2}TGPO: Agentic Turn-Group Policy Optimization with Adaptive Turn-level Clipping
}

\author{%
  Dingwei Chen{$^{\spadesuit\Diamond}$\thanks{~~Work was done during the internship at Tencent Inc.}}, Zefang Zong{$^{\spadesuit}$}, Zhipeng Ma{$^{\spadesuit}$}, Leo Luo{$^{\spadesuit}$}, Yang Li{$^{\spadesuit}$} \\
  \textbf{Chengming Li}{$^{\heartsuit}\thanks{~~Corresponding author.}$}, \textbf{Peng Chen}{$^{\spadesuit}$}, \textbf{Jie Jiang}{$^{\spadesuit}\footnotemark[2]$} \\
  $^\spadesuit$Tencent Inc   $^\Diamond$The Chinese University of Hong Kong  $^\heartsuit$Shenzhen MSU-BIT University\\
  \texttt{cuso4cdw@gmail.com}, \texttt{licm@smbu.edu.cn} \\
  \texttt{\{willzong,thomasyngli\}@tencent.com} \\
  \faGithub \href{https://github.com/CuSO4-Chen/A-TGPO}{\textcolor{cyan}{\footnotesize{\underline{\texttt{CuSO4-Chen/A-TGPO}}}}}
}

\begin{document}

\maketitle

\begin{abstract}

Reinforcement learning for agentic large language models (LLMs) typically relies on a sparse, trajectory-level outcome reward, making it difficult to evaluate the contribution of individual tool-calls within multi-turn interactions. Existing approaches to such process credit assignment either depend on separate external process reward models that introduce additional consumption, or tree-based structural rollout that merely redistributes the outcome signal while constraining trajectory diversity. A promising alternative leverages the per-turn change in the policy's predicted probability of the ground-truth, termed Information Gain (IG), as an intrinsic process signal without an external evaluator. However, prior work on leveraging IG signals within the RL training loop faces three systematic challenges: normalizing across turns that face heterogeneous positional contexts can distort the relative standing of individual turns, accumulating a variable number of terms causes advantage magnitudes to drift with trajectory depth, and a fixed clipping range governs policy updates identically for turns with vastly different IG signals. In this paper, we propose \textbf{A\textsuperscript{2}TGPO} (\textbf{A}gentic \textbf{T}urn-\textbf{G}roup \textbf{P}olicy \textbf{O}ptimization with \textbf{A}daptive Turn-level Clipping), which retains IG as the intrinsic signal but re-designs how it is normalized, accumulated, and consumed: (i) \textbf{turn-group normalization}: normalizes IG within each (prompt, turn-index) group so that each turn is compared only against peers at the same interaction depth; (ii) \textbf{variance-rescaled discounted accumulation}: divides cumulative normalized IG by square root of accumulated terms to keep advantage magnitudes comparable across turn positions; and (iii) \textbf{adaptive turn-level clipping}: modulates each turn's clipping range based on its normalized IG, widening the update region for informative turns and narrowing it for uninformative ones. On seven single-hop and multi-hop QA benchmarks across three backbones, A\textsuperscript{2}TGPO consistently outperforms prior strong baselines, improving over existing RL methods by $+1.75$ on multi-hop and $+1.69$ on single-hop on average.
\end{abstract}

\section{Introduction}

Agentic large language models (Agentic LLMs) that utilize external tools to engage in multi-turn interactions have demonstrated strong capabilities in complex tasks such as web navigation, code generation, and open-domain question answering~\cite{yao2023react,li2025miv,li2026make,jin2025search}. To enhance an agent's tool-use ability, reinforcement learning (RL) has emerged as a powerful paradigm. Driven by the success of rule-based outcome verification in LLM reasoning~\cite{guo2025deepseek,yu2025dapo,zheng2025groupsequencepolicyoptimization}, this critic-free approach is naturally extended to agentic settings~\cite{jin2025search, wang2025ragen}. As agentic rollouts further introduce a turn-based interaction structure that interleaves model-generated reasoning with tool responses, a range of specialized optimization is designed for this paradigm from trajectory selection and rollout mechanism ~\cite{dong2025agenticarpo,dong2025agenticaepo,feng2025group,zong2026at2poagenticturnbasedpolicy,shen2025carl,wu2025webdancer}. Despite these efforts, they still drive the policy with a single trajectory-level outcome, providing no mechanism to distinguish tool-calls that genuinely advance toward the answer from those that merely prolong the interaction. Addressing this limitation requires a process signal that can evaluate each turn's role in advancing toward the final answer, which is also known as process credit assignment.

\begin{figure}[t]
  \centering
  \includegraphics[width=0.9\linewidth]{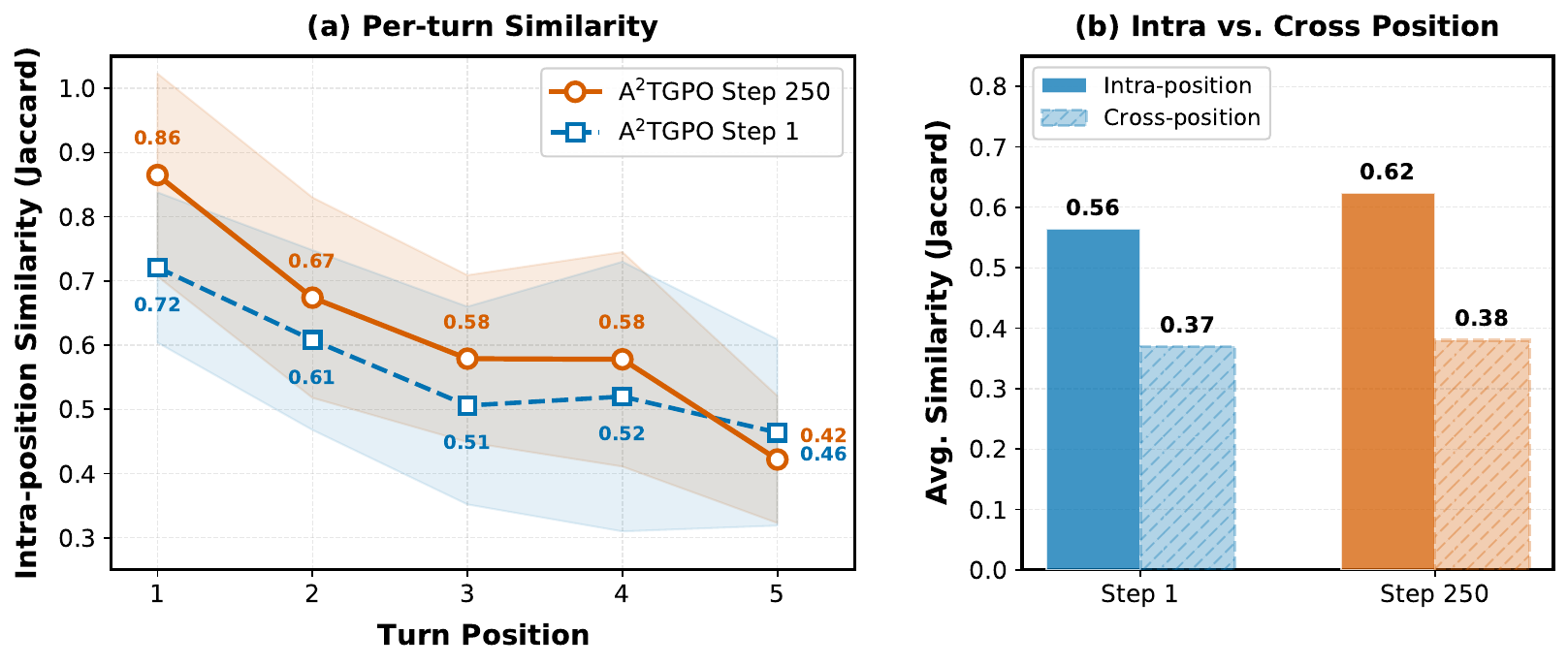}
  \vspace{-4.7mm}
  \caption{\textit{Left}: Per-turn intra-position context similarity between rollouts of the same prompt. \textit{Right}: Overall intra-position vs.\ cross-position similarity. Rollouts at the same turn share substantially more similar contexts than those at different turns.}
  \label{fig:context-similarity}
  \vspace{-4.7mm}
\end{figure}

Existing routes to such per-turn supervision fall into three categories. Process reward models (PRMs)~\citep{lightman2023let,wang2024math} score process steps to supply dense reward signals but require a separately trained external evaluator and carry non-trivial risks of reward hacking. Tree-based methods~\citep{hou2025treerl,ji2025treesearchllmagent,yang2025treerpo} reorganize rollouts into shared-prefix trees and redistribute the outcome reward across branches, eliminating the external evaluator but merely reallocating the outcome signal while constraining trajectory diversity. Regarding these limitations, a third line of work derives per-turn credit from model-intrinsic signals without external evaluators. GiGPO~\citep{feng2025group} assigns a group-relative advantage to turns sharing the same state across trajectories, though identifying equivalent states in open-ended generation remains challenging. Along this line, recent work further proposes to measure the change in the policy's predicted probability of the ground-truth answer after each turn, termed \emph{Information Gain} (IG), as an intrinsic per-turn process signal. For example, IGPO~\citep{wang2026information} normalizes IG signals across all turns and derives turn-level advantages through discounted accumulation.

However, prior work on leveraging IG as a per-turn process signal in the RL training loop faces three systematic challenges. First, normalizing IG across all turns of all rollouts sharing a prompt pools turn positions that face fundamentally different contexts, overlooking the incomparability of information gains computed under heterogeneous states and distorting the relative standing of individual turns. Second, a discounted cumulative advantage that sums a variable number of normalized IG terms along the trajectory causes advantage magnitudes to vary inconsistently with trajectory depth rather than remaining on a comparable scale across turn positions. Third, a fixed clipping range governs policy updates identically for turns with vastly different IG signals, preventing the optimizer from modulating update intensity according to per-turn informativeness.

To address these challenges, we propose \textbf{A\textsuperscript{2}TGPO} (\textbf{A}gentic \textbf{T}urn-\textbf{G}roup \textbf{P}olicy \textbf{O}ptimization with \textbf{A}daptive Turn-level Clipping), which retains IG as the intrinsic per-turn signal but re-designs how it is normalized, accumulated, and consumed by the policy optimization. Our key observation is that the \emph{turn-index} provides a natural unit for both normalization and credit assignment in agentic rollouts: trajectories sharing the same prompt and having executed the same number of interactions tend to be in similar contexts and states, especially before trajectories branch substantially at early turns, and thus form a meaningful comparison group. We empirically verify this in Figure~\ref{fig:context-similarity}: rollouts at the same turn position share high contextual similarity that decreases with depth (left; 0.86 at turn~1 declining to 0.42 at turn~5), and overall intra-position similarity substantially exceeds cross-position similarity (right; 0.62 vs.\ 0.38), confirming that turn-group comparison is both natural and well-founded (detailed analysis in Appendix~\ref{sec:context-similarity}). Building on this, A\textsuperscript{2}TGPO introduces three components. To resolve the incomparability caused by pooled normalization, we design \emph{\textbf{IG-based turn-group normalization}} that normalizes IG within each (prompt, turn-index) group so that each turn is evaluated only against peers at the same interaction depth. To stabilize the scale of discounted cumulative advantages across trajectory depths, we propose \emph{\textbf{discounted cumulative advantage with variance rescaling}} that divides the cumulative normalized IG by the square root of the number of accumulated terms, keeping advantage magnitudes comparable across turn positions. To achieve the adaptive policy update, we further introduce \emph{\textbf{IG-based adaptive turn-level clipping}} that re-uses the normalized IG to modulate each turn's clipping range, widening the clipping range for informative turns and narrowing it for uninformative ones. Furthermore, A\textsuperscript{2}TGPO operates the importance-sampling ratio and clipping at the turn level rather than the token or sequence level, aligning the optimization granularity with the natural interaction structure of agentic rollouts. Our main contributions are summarized as follows:

\begin{itemize}[leftmargin=*] 
      \item We design a turn-group normalization scheme that normalizes IG within each $(q, t)$ group. Each turn is evaluated only against positional peers at the same interaction depth, eliminating the incomparability inherent in pooled normalization.
      \item We propose a variance-rescaled discounted accumulation to keep advantage magnitudes comparable across turn positions, and further introduce an adaptive turn-level clipping mechanism that adaptively modulates each turn's clip range based on its normalized IG.
      \item We evaluate A\textsuperscript{2}TGPO on seven single-hop and multi-hop open-domain QA benchmarks across three backbones and show that it consistently outperforms prior strong baselines, improving over existing RL methods by $+1.75$ on multi-hop and $+1.69$ on single-hop on average.
\end{itemize}

\section{Related Work}

\textbf{Reinforcement Learning in LLMs and Agents.} Reinforcement learning has become a cornerstone for enhancing LLM reasoning and alignment~\cite{yu2025dapo,zheng2025groupsequencepolicyoptimization,jin2025search,wang2025ragen}. Building on PPO-based RLHF~\cite{schulman2017proximal,christiano2017deep,stiennon2020learning}, recent critic-free methods such as GRPO~\cite{guo2025deepseek} and DAPO~\cite{yu2025dapo} estimate advantages from group-relative comparisons and progressively refine the clipping granularity~\cite{zheng2025groupsequencepolicyoptimization} with their verifiable-reward~\cite{shao2024deepseekmathpushinglimitsmathematical, guo2025deepseek}. Extending this foundation, a growing line of work tailors optimization to the agentic paradigm. Search-R1~\cite{jin2025search} integrates search actions into the RL loop and establishes an early template for tool-augmented agent training. ARPO~\cite{dong2025agenticarpo} exploits the entropy spike that follows tool responses to trigger selective rollout branching at uncertain decision points, while AEPO~\cite{dong2025agenticaepo} further curbs over-branching through entropy-balanced sampling and updates. Although these outcome-driven methods leverage sampling dynamics or loss optimization to improve agentic training, they still rely on a single trajectory-level reward, leaving per-turn evaluation largely unresolved.

\textbf{Credit Assignment in Agentic Reinforcement Learning.} Outcome-driven agentic RL typically provides only a sparse, trajectory-level reward, which is too coarse to assign credit across long multi-turn interactions. One line of work addresses this via process reward models (PRMs) that score process steps to supply dense reward~\cite{lightman2023let, wang2024math, setlur2025rewarding, choudhury2025process}, but such approaches require a separately loaded reward model. Another route organizes rollouts into tree structures and redistributes process credit across shared prefixes and branches~\cite{hou2025treerl, xie2024monte, yang2025treerpo, ji2025treesearchllmagent}. While concurrent efforts further improve the exploration dynamics through entropy-guided branch expansion~\cite{shen2025carl, zong2026at2poagenticturnbasedpolicy}, this paradigm simply reallocates the outcome reward among nodes and constrains the diversity of trajectories. Besides the two paradigms above, a third line of work designs intrinsic signals without external evaluators~\citep{feng2025group,wang2026information}. GiGPO~\cite{feng2025group} introduces a hierarchical grouping scheme that pools same-state actions across trajectories to yield finer-grained credit. IGPO~\cite{wang2026information} quantifies per-turn information gain signals to estimate the advantage of each tool call. However, these methods still lack an objective comparison for estimation across turns, making it difficult to calibrate the relative importance of individual tool calls.


\section{Preliminaries}
\label{sec:prelim-igpo}

\textbf{Task Definition.} Following the agentic RL formulation of prior
work~\citep{jin2025search, dong2025agenticarpo, zong2026at2poagenticturnbasedpolicy, wang2026information}, a language model policy $\pi_\theta$ answers a query $q$ through multi-turn interaction with a tool environment $\mathcal{E}$. Given a dataset $\mathcal{D} = \{(q, a^\star)\}$, the agent produces a rollout $\tau \sim \pi_\theta(\cdot \mid q, \mathcal{E})$ concluding with a prediction $\hat{a}$ and receives a trajectory-level reward $R(\tau)$ measuring the correctness of $\hat{a}$ against $a^\star$. The learning objective is as follows:
  \begin{equation}
    \mathcal{J}(\pi_\theta)
    = \mathbb{E}_{q \sim \mathcal{D}}\,
      \mathbb{E}_{\tau \sim \pi_\theta(\cdot \mid q,\, \mathcal{E})}
      \bigl[R(\tau)\bigr].
    \label{eq:task-objective}
  \end{equation}
\textbf{Multi-turn Rollout.} Following the ReAct paradigm~\citep{yao2023react}, at each turn $t \in \{1, \ldots, T\}$ the policy samples a model-generated segment $y_t$, and the environment returns an observation $o_t$ when $y_t$ is a tool call. The entire trajectory is $\tau = (y_1, o_1, \ldots, y_{T-1}, o_{T-1}, y_T)$, with probability $\pi_\theta(\tau \mid q) = \prod_{t=1}^{T} \pi_\theta\!\bigl(y_t \mid q, y_{<t}, o_{<t}\bigr)$. Only tokens in $\{y_t\}_{t=1}^{T}$ contribute to the policy gradient; tokens in the observations $\{o_t\}$ are produced by $\mathcal{E}$ and masked out in the loss calculation. For each query $q$, $G$ trajectories $\{\tau_i\}_{i=1}^{G}$ are sampled from $\pi_{\theta_{\text{old}}}$ for group-based policy optimization~\citep{shao2024deepseekmathpushinglimitsmathematical, guo2025deepseek}.

\textbf{Turn-level Information Gain.} Since a trajectory-level reward $R(\tau_i)$ conveys little information about the value of the individual $T_i$ turns or their respective contributions to the final outcome, previous work~\cite{wang2026information} introduces a turn-level signal by quantifying the change in the policy's assigned probability of the ground-truth answer $a$ at each turn. Let $\tau_{i,\leq t}=(y_{i,1}, o_{i,1}, \dots, y_{i,t}, o_{i,t})$ denote the prefix of $\tau_i$ through turn $t$. The length-normalized conditional probability of $a=(a_1,\dots,a_L)$ is
\begin{equation}
  \pi_\theta\!\left(a \mid q,\,\tau_{i,\leq t}\right) \;=\; \exp\!\left(\frac{1}{L}\sum_{j=1}^{L}\log \pi_\theta\!\left(a_j \mid q,\,\tau_{i,\leq
  t},\,a_{<j}\right)\right),
  \label{eq:teacher-forced-prob}
\end{equation}
and the \emph{information gain} of turn $t$ is defined as follows:
\begin{equation}
  \mathrm{ig}_{i,t} \;=\; \pi_\theta\!\left(a \mid q,\,\tau_{i,\leq t}\right) \;-\; \pi_\theta\!\left(a \mid q,\,\tau_{i,\leq t-1}\right), \qquad 1
   \leq t < T_i.
  \label{eq:ig-def}
\end{equation}
The signal $\mathrm{ig}_{i,t}$ is computed from the policy's own likelihoods, treated as a stop-gradient quantity, and adds one forward pass over $a$ per turn. For $t = 1$, $\mathrm{ig}_{i,1}$ measures the gain from the first tool call relative to the query-only baseline $\pi_\theta(a \mid q)$.

\textbf{Policy Update with Turn-level Advantages.} Following the example of IGPO, it assembles a reward vector $\mathbf{r}_i=(r_{i,1},\dots,r_{i,T_i})$ with $r_{i,t}=\mathrm{ig}_{i,t}$ for $t<T_i$ and $r_{i,T_i}=R(\tau_i)$. All turn rewards in the group are as follows:
\begin{equation}
  \mathcal{R} \;=\; \{\,r_{i,t} : i = 1,\dots,G;\; t = 1,\dots,T_i\,\}
  \label{eq:igpo-group}
\end{equation}
and then are jointly $z$-normalized and propagated backward through discounted accumulation as follows:
\begin{equation}
  \tilde{A}_{i,t} \;=\; \sum_{k=t}^{T_i} \gamma^{\,k-t}\,\frac{r_{i,k} - \mathrm{mean}(\mathcal{R})}{\mathrm{std}(\mathcal{R})},
  \label{eq:igpo-adv}
\end{equation}
where $\gamma\in(0,1]$ is a discount factor. $\tilde{A}_{i,t}$ replaces the trajectory-level advantage in the standard clipped policy objective, so that process turns receive finer-grained credit than GRPO, while the clipping range $\epsilon$ remains fixed across all turns and all samples.

\section{Methodology}
\begin{figure}[t]
    \centering
    \includegraphics[width=1.0\linewidth]{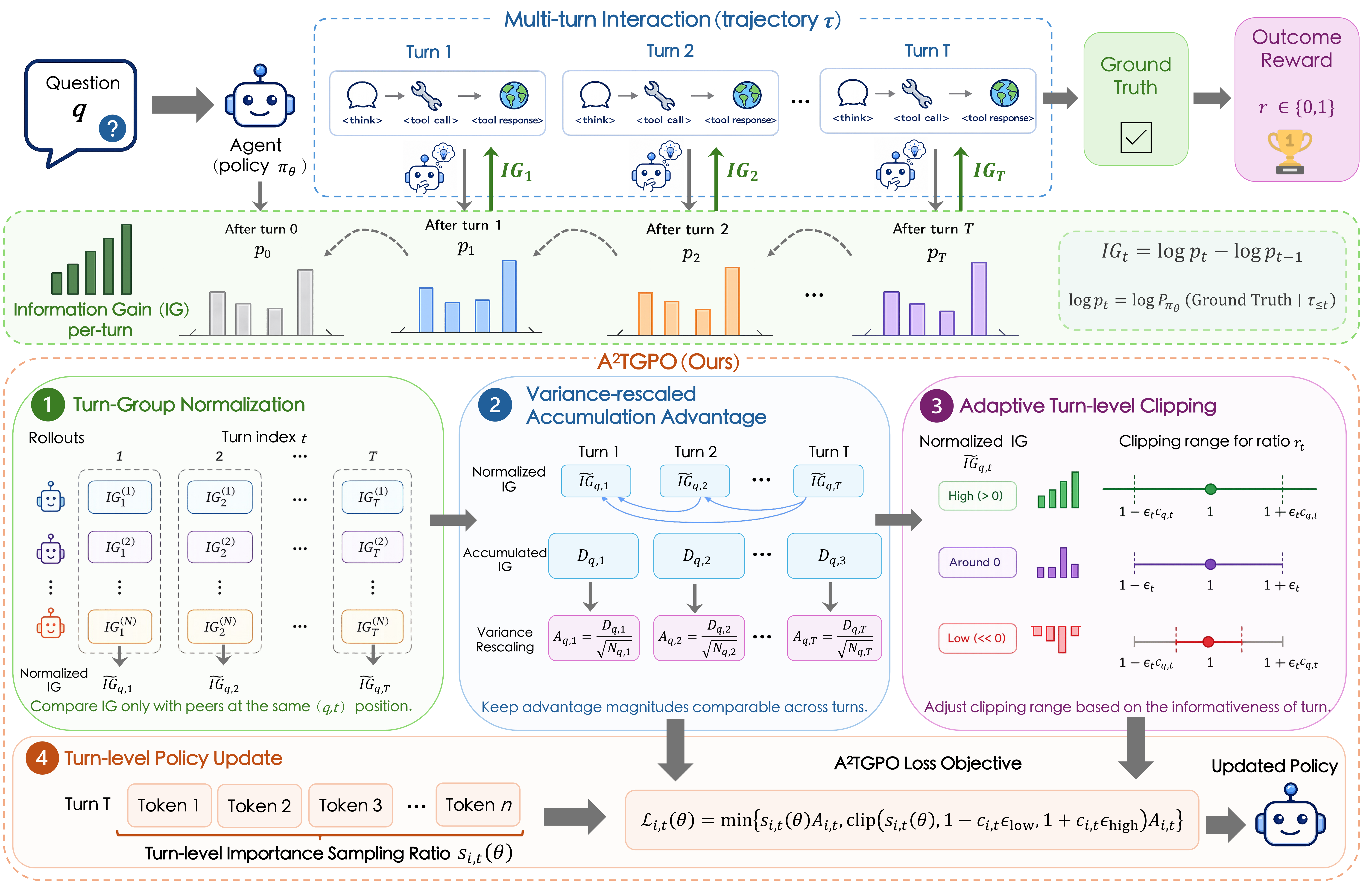}
    \vspace{-4.7mm}
    \caption{The framework of A\textsuperscript{2}TGPO. Raw IG signals are first normalized within each turn group, then flow into discounted accumulation with variance rescaling to produce the turn-level advantage $\widehat{A}_{i,t}$, while a sigmoid mapping yields the adaptive clip scale $c_{i,t}$. Both are consumed by the turn-level clipped policy loss.}
    \label{fig:framework}
    \vspace{-4.5mm}
  \end{figure}

This section presents A\textsuperscript{2}TGPO, building on the IG-based paradigm introduced in Section~\ref{sec:prelim-igpo}. An overview of the framework is shown in Figure~\ref{fig:framework}. The following subsections illustrate the three components in turn.

\subsection{IG-based Turn-Group Normalization}
\label{sec:method-tgnorm}

Given the per-turn information gain $\mathrm{ig}_{i,t}$ computed following the same procedure as IGPO (Eq.~\eqref{eq:ig-def}), A\textsuperscript{2}TGPO normalizes each per-turn information gain against a group of peers that share both the prompt $q$ and the specific turn index $t$. For each prompt $q$ and each turn index $t$, we define the turn-group as follows:
\begin{equation}
  \mathcal{G}_{q,t} \;=\; \{\, \mathrm{ig}_{i,t} \;:\; i = 1,\dots,G,\; t \leq T_i\,\},
  \label{eq:tg-group}
\end{equation}
where rollouts that complete before reaching turn $t$ do not contribute to $\mathcal{G}_{q,t}$. Since rollout lengths vary, $|\mathcal{G}_{q,t}|$ decreases with $t$; when $|\mathcal{G}_{q,t}| \leq 1$, we set $\widehat{\mathrm{ig}}_{i,t} = 0$, relying solely on the outcome reward for that turn (see Appendix~\ref{sec:theory-tg-robust} for a robustness analysis). The turn-group normalized information gain is then defined with $z$-normalization:
\begin{equation}
  \widehat{\mathrm{ig}}_{i,t} \;=\; \frac{\mathrm{ig}_{i,t} - \mathrm{mean}(\mathcal{G}_{q,t})}{\mathrm{std}(\mathcal{G}_{q,t})}.
  \label{eq:tg-normed}
\end{equation}
Grouping by $(q, t)$ reflects the empirical observation in agentic settings that, trajectories sharing the same prompt and having executed the same number of interactions tend to be in similar contexts and states, especially before trajectories branch substantially at early turns. The pooled normalization in Eq.~\eqref{eq:igpo-group}, however, computes a single mean and variance across all turn positions, conflating signals from inherently different regimes: early turns operate on minimal evidence while later turns condition on accumulated tool responses, so their information-gain distributions already differ in both location and scale. This mismatch is compounded by the chain-like dependence of information gains: a tool call that returns highly supportive content absorbs much of the available information, may systematically lower the expected gain at subsequent turns even when those turns are themselves effective. Because such pooling distorts the relative standing of individual turns, A\textsuperscript{2}TGPO normalizes within each $(q,t)$ group instead, evaluating each turn against peers that share its position and capturing what constitutes a superior or inferior tool call \emph{at that specific position}. The normalized signal $\widehat{\mathrm{ig}}_{i,t}$ is dimensionless and position-conditional, and serves as the turn-level input to the advantage construction developed in Section~\ref{sec:method-varscale}.
\subsection{Discounted Cumulative Advantage with Variance Rescaling}
\label{sec:method-varscale}

With the normalized information gains $\widehat{\mathrm{ig}}_{i,t}$ from Eq.~\eqref{eq:tg-normed}, we construct a turn-level advantage that propagates per-turn credit backward along the trajectory while equalizing the scale of early-turn and late-turn contributions. For each turn $t \in \{1,\dots,T_i-1\}$ of trajectory $\tau_i$ (except $T_i$ for final answer), the backward cumulative information gain is defined as follows:
\begin{equation}                                          
  D_{i,t} \;=\; \sum_{k=t}^{T_i-1} \gamma^{\,k-t}\,\widehat{\mathrm{ig}}_{i,k},
  \qquad                                                       
  n_{i,t} \;=\; T_i - t,
  \label{eq:cumulative-ig}
\end{equation}          
To further capture the long-horizon dependencies, $D_{i,t}$ accumulates the normalized signals from all downstream turns within the same trajectory, propagating credit backward from later turns toward earlier ones, where $\gamma\in(0,1]$ is a discount factor that down-weights distant turns. In the baseline formulation (Eq.~\eqref{eq:igpo-adv}), the discounted cumulation sums a variable number of terms across turn positions, causing advantage magnitudes to vary inconsistently with trajectory depth. Since the variance of the sum grows linearly in $n_{i,t}$ under mild independence assumptions, rescaling by $\sqrt{n_{i,t}}$ yields approximately constant variance across all turn positions, keeping advantage magnitudes comparable regardless of trajectory depth (see Appendix~\ref{sec:theory-variance} for a formal derivation).

This per-turn credit is combined with the outcome reward to enhance the outcome orientation. Let $\widehat{R}_i$ denote the outcome reward $R(\tau_i)$ after per-prompt GRPO-like normalization across the $G$ trajectories sharing $q$. The turn-level advantage used by A\textsuperscript{2}TGPO is computed as
\begin{equation}
  \widehat{A}_{i,t} \;=\;
  \begin{cases}
  \dfrac{D_{i,t}}{\sqrt{n_{i,t}}} \;+\; \widehat{R}_i, & 1 \leq t \leq T_i - 1, \\[6pt]
  \widehat{R}_i, & t = T_i,
  \end{cases}
  \label{eq:atgpo-adv}
\end{equation}
where the final answer turn $T_i$ conveys no defined information gain and inherits only the outcome signal, while process turns combine the rescaled backward cumulative credit with the outcome term. 

\subsection{IG-based Adaptive Turn-level Clipping}
\label{sec:method-adaclip}

We refine the clipping range of the policy loss on a per-turn basis, using the normalized information gain $\widehat{\mathrm{ig}}_{i,t}$ from Eq.~\eqref{eq:tg-normed} to assign the policy a wider update range on turns that yield higher information gain while a narrower range on turns where the gain is low or negative. Furthermore, we adopt turn-level policy optimization instead of token- or sequence- level in previous work~\cite{guo2025deepseek, zheng2025groupsequencepolicyoptimization} to align the optimization objectives with the turn-based interaction structure of agentic LLMs.

Concretely, for turn $t$ of rollout $\tau_i$, the turn-level importance-sampling ratio is computed as the length-normalized geometric mean of the per-token ratios:
\begin{equation}
  s_{i,t}(\theta) \;=\; \exp\!\left(\frac{1}{|y_{i,t}|}\sum_{k=1}^{|y_{i,t}|} \log\frac{\pi_\theta(y_{i,t,k} \mid \cdot)}{\pi_{\theta_{\mathrm{old}}}(y_{i,t,k} \mid \cdot)}\right),
  \label{eq:turn-is}
\end{equation}
where $|y_{i,t}|$ is the number of generated tokens in turn $t$. The ratio $s_{i,t}(\theta)$ is shared by all tokens within the same turn. The effective clipping range of $s_{i,t}(\theta)$ is then gated by a sigmoid of $\widehat{\mathrm{ig}}_{i,t}$ as follows:
\begin{equation}                                                   
  c_{i,t} \;=\; 1 + \beta\left(2\sigma(\widehat{\mathrm{ig}}_{i,t}) - 1\right),
  \label{eq:clip-scale}                                                        
\end{equation}       
where $\sigma$ is the logistic sigmoid and $\beta \in [0, 1)$ is a hyperparameter that controls the maximum relative deviation of the clipping range from its base value. Specifically, the scale factor $c_{i,t}$ is monotonically increasing in $\widehat{\mathrm{ig}}_{i,t}$ and bounded within $(1-\beta,\,1+\beta)$: turns with higher information gain in rank receive a wider clipping range while turns with lower or negative gain receive a narrower one.                   
                                               
Inspired by DAPO~\citep{yu2025dapo}, we use $(\epsilon_{\mathrm{low}}, \epsilon_{\mathrm{high}})$ to denote the base asymmetric clipping bounds. The effective per-turn bounds used by A\textsuperscript{2}TGPO are $c_{i,t}\,\epsilon_{\mathrm{low}}$ and $c_{i,t}\,\epsilon_{\mathrm{high}}$. The A\textsuperscript{2}TGPO loss is defined by substituting the turn-level ratio from Eq.~\eqref{eq:turn-is} and the turn-level advantage $\widehat{A}_{i,t}$ from Eq.~\eqref{eq:atgpo-adv} into the clipped policy objective, as follows:                \begin{equation}
    \begin{aligned}
    \mathcal{L}_{\mathrm{A^2TGPO}}(\theta)
    &= -\,\mathbb{E}_{q,\{\tau_i\}_{i=1}^{G}}\!\Bigg[ \frac{1}{G}\sum_{i=1}^{G}\frac{1}{|\mathcal{M}(\tau_i)|}\sum_{(t,k) \in \mathcal{M}(\tau_i)} \min\!\Big( s_{i,t}(\theta)\, \widehat{A}_{i,t}, \\
    &\qquad\qquad\qquad\quad \mathrm{clip}\!\Big(s_{i,t}(\theta),\;1 - c_{i,t}\epsilon_{\mathrm{low}},\;1 + c_{i,t}\epsilon_{\mathrm{high}}\Big) \widehat{A}_{i,t} \Big) \Bigg],
    \end{aligned}
    \label{eq:atgpo-loss}
\end{equation}
where $|\mathcal{M}(\tau_i)|$ is the total number of model-generated tokens in $\tau_i$. The advantage $\widehat{A}_{i,t}$ is shared by all tokens within turn $t$, and $c_{i,t}$ enters Eq.~\eqref{eq:atgpo-loss} only as a scaling factor on the clipping bounds, contributing no gradient with respect to $\theta$ since $\widehat{\mathrm{ig}}_{i,t}$ is a stop-gradient quantity (Section~\ref{sec:prelim-igpo}).

\section{Experiments}

\subsection{Experiment Settings}

\textbf{Datasets.}  We evaluate A\textsuperscript{2}TGPO in a tool-integrated search setting and leverage the retrieval environment following Search-R1~\citep{jin2025search}, which designs a local search engine as an external tool during both training and evaluation. Seven open-domain question answering benchmarks are used, organized into two groups by reasoning depth. \textbf{Multi-hop} benchmarks consist of HotpotQA~\citep{yang2018hotpotqa}, 2WikiMultihopQA~\citep{ho2020constructing}, MuSiQue~\citep{trivedi2022musique}, and Bamboogle~\citep{press2023measuring}. \textbf{Single-hop} benchmarks consist of Natural Questions (NQ)~\citep{kwiatkowski2019natural}, TriviaQA~\citep{joshi2017triviaqa}, and PopQA~\citep{mallen2023llm_memorization}. We train and evaluate on three backbones: Qwen3-4B, Qwen3-8B, and Qwen2.5-7B. We report \textbf{Exact Match (EM)} as the primary metric on every benchmark as well as the average accuracy across all evaluation samples. This experiment setting deliberately avoids proprietary APIs and heavyweight tool infrastructure, keeping the evaluation reproducible and concentrating on the progress of the RL algorithm.

\textbf{Baselines.} We first include ReAct~\citep{yao2023react} as a non-RL reference that prompts the backbone to interleave reasoning and tool calls without training. Furthermore, we compare A\textsuperscript{2}TGPO against a range of RL methods spanning recent advances in policy optimization and agentic training. The first part consists of several widely-used RLVR baselines: GRPO~\citep{guo2025deepseek}, DAPO~\citep{yu2025dapo} and GSPO~\citep{zheng2025groupsequencepolicyoptimization}. Another part consists of several recently promising agentic RL baselines: Tree-GRPO~\citep{ji2025treesearchllmagent}, GiGPO~\citep{feng2025group}, IGPO~\citep{wang2026information}, AEPO~\citep{dong2025agenticaepo}. Similar to previous work~\cite{zong2026at2poagenticturnbasedpolicy}, we also observed during our reproduction that Tree-GRPO frequently crashed during training on Qwen3 family. We report its results on Qwen2.5-7B. Note that we will present the details of baselines and our implementation in Appendix~\ref{sec:implementation}.

\begin{table}[t]
  \centering                  
  \footnotesize   
  \setlength{\tabcolsep}{5pt}
  \renewcommand{\arraystretch}{1.0}
  \vspace{-1.5mm}
  \caption{Experiment results on three backbone models across seven datasets. The \textbf{bolded} values indicate the best result. Our proposed A\textsuperscript{2}TGPO outperforms existing methods in most cases.}
  \label{tab:main_results}
  \begin{tabular}{l|ccccc|cccc}
  \toprule
   \multirow{2}{*}{\textbf{Method}} & \multicolumn{5}{c|}{\textbf{Multi-Hop}} & \multicolumn{4}{c}{\textbf{Single-Hop}} \\
  \cmidrule(lr){2-6}\cmidrule(lr){7-10}
   & Hotpot & 2Wiki & MuSiQue & Bamb. & Avg.                                                                                      
                 & NQ & TriviaQA & PopQA & Avg. \\
  \midrule                                                                                                                        
  \multicolumn{10}{c}{\textit{Backbone Model: Qwen3-4B}} \\
  \midrule
  ReAct        & 30.42 & 32.92 & 12.83 & 44.80 & 30.01 & 26.75 & 53.53 & 35.34 & 41.31 \\
  + GRPO         & 44.76 & 51.40 & 21.63 & 50.40 & 46.02 & 45.98 & 65.17 & 49.18 & 54.97 \\
  + DAPO         & 45.95 & \textbf{51.81} & 21.68 & 51.20 & 46.65 & 47.50 & 65.84 & \textbf{51.03} & 56.33 \\
  + GSPO         & 47.07 & 49.25 & 22.68 & 49.60 & 45.69 & 46.01 & 64.24 & 48.50 & 54.28 \\
  + GiGPO        & 45.68 & 51.08 & 22.74 & 51.20 & 46.26 & 46.20 & 65.70 & 49.55 & 55.40 \\
  + IGPO         & 45.10 & 49.23 & 21.28 & 50.40 & 44.88 & 46.05 & 65.25 & 49.38 & 55.12 \\
  + AEPO         & 46.36 & 51.78 & 23.47 & 52.00 & 46.96 & 45.71 & 64.66 & 50.13 & 55.20 \\
  + \textbf{A\textsuperscript{2}TGPO (Ours)}    &
                \textbf{49.42} & 51.29 & \textbf{25.21} & \textbf{53.60} & \textbf{48.06} & \textbf{47.50} & \textbf{66.65} &
  50.60 & \textbf{56.44} \\
  \midrule
  \multicolumn{10}{c}{\textit{Backbone Model: Qwen3-8B}} \\
  \midrule
  ReAct        & 20.66 & 19.05 & 9.56  & 37.60 & 18.66 & 21.16 & 41.81 & 27.37 & 32.19 \\
  + GRPO         & 45.75 & 52.40 & 20.15 & 52.80 & 46.76 & 44.40 & 66.12 & 48.87 & 55.00 \\
  + DAPO         & 48.30 & 52.55 & 22.68 & 54.40 & 47.96 & 50.60 & \textbf{67.60} & 50.50 & 57.14 \\
  + GSPO         & 48.30 & 51.24 & 22.98 & 53.60 & 47.26 & 44.20 & 66.42 & 48.36 & 54.85 \\
  + GiGPO        & 48.55 & 51.68 & 22.66 & 54.40 & 47.56 & 47.85 & 66.15 & 50.22 & 56.10 \\
  + IGPO         & 46.00 & 52.50 & 21.77 & 52.80 & 47.07 & 44.80 & 66.40 & 49.10 & 55.27 \\
  + AEPO         & 47.92 & 51.68 & 22.72 & 55.20 & 47.36 & 48.72 & 67.11 & 50.57 & 56.75 \\
  + \textbf{A\textsuperscript{2}TGPO (Ours)}    &
               \textbf{49.85} & \textbf{53.06} & \textbf{23.55} & \textbf{56.00} & \textbf{48.85} & \textbf{50.80} & 66.88 & \textbf{51.69}
   & \textbf{57.54} \\
  \midrule
  \multicolumn{10}{c}{\textit{Backbone Model: Qwen2.5-7B}} \\
  \midrule
  ReAct        & 2.85  & 1.94  & 0.58  & 4.00  & 2.10  & 4.34  & 10.67 & 9.32  & 9.23 \\
  + GRPO         & 46.64 & 45.60 & 19.97 & 45.90 & 43.19 & 44.35 & 63.65 & 48.70 & 53.96 \\
  + DAPO         & 46.15 & 46.57 & 19.95 & 42.40 & 43.54 & 50.85 & 63.60 & 48.61 & 54.70 \\
  + GSPO         & 46.03 & 46.00 & 19.03 & 41.60 & 43.09 & 48.33 & 61.55 & 48.45 & 53.51 \\
  + Tree-GRPO    & 41.05 & 40.65 & 18.77 & 40.80 & 38.44 & 46.27 & 61.40 & 43.46 & 50.76 \\
  + GiGPO        & 46.10 & 46.38 & 19.68 & 43.20 & 43.40 & 47.45 & 62.68 & 48.78 & 54.00 \\
  + IGPO         & 46.70 & 45.85 & 20.66 & 45.90 & 43.43 & 44.70 & 63.90 & 49.00 & 54.24 \\
  + AEPO         & 45.80 & 46.28 & 19.75 & 44.00 & 43.26 & 47.70 & 62.83 & 48.91 & 54.16 \\
  + \textbf{A\textsuperscript{2}TGPO (Ours)}    &
               \textbf{49.56} & \textbf{47.74} & \textbf{23.63} & \textbf{48.00} & \textbf{45.75} & \textbf{51.05} &
  \textbf{64.85} & \textbf{50.64} & \textbf{56.25} \\
  \bottomrule
  \end{tabular}
  \vspace{-1.5mm}
\end{table}

\begin{table}[t]                                             
    \centering
    \footnotesize                         
    \setlength{\tabcolsep}{5pt}
    \renewcommand{\arraystretch}{1.0}
    \vspace{-1.5mm}
    \caption{Ablation study on A\textsuperscript{2}TGPO components on multi-hop benchmarks based on Qwen3-4B.}
    \label{tab:ablation_qwen3_4b_multihop}
    \begin{tabular}{l|ccccc}
    \toprule
    \textbf{Method} & Hotpot & 2Wiki & MuSiQue & Bamb. & Avg. \\
    \midrule
    GRPO                                   & 44.76 & 51.40 & 21.63 & 50.40 & 46.02 \\
    \midrule
    IGPO                                   & 45.10 & 49.23 & 21.28 & 50.40 & 44.88 \\
    \;+ TG-Norm                            & 47.24 & 50.17 & 22.68 & 52.40 & 46.27 \\
    \;+ TG-Norm + $D_t$-rescaling         & 47.94 & 50.54 & 22.77 & 52.00 & 46.71 \\
    \;+ TG-Norm + $D_t$-rescaling + Ada-Clipping \textbf{(A\textsuperscript{2}TGPO)}      & \textbf{49.42} & \textbf{51.29} & \textbf{25.21} & \textbf{53.60} & \textbf{48.06} \\
    \bottomrule
    \end{tabular}
    \vspace{-4.5mm}
\end{table}

\subsection{Main Results of A\textsuperscript{2}TGPO}

Table~\ref{tab:main_results} reports results across three backbones and seven benchmarks. A\textsuperscript{2}TGPO achieves the highest sample-weighted average (\textit{Avg.}) on all benchmark settings, with gains that are consistently larger on multi-hop benchmarks where longer tool-use trajectories amplify the benefit of per-turn credit assignment. A\textsuperscript{2}TGPO improves on average over existing RL methods by $+1.75$  on multi-hop and $+1.69$ on single-hop. The improvements hold across model categories and scales, confirming the generalization of proposed method. Among the baselines, DAPO is the strongest baseline on single-hop owing to its higher clip and dynamic sampling that effectively triples the per-step rollout budget, yet A\textsuperscript{2}TGPO matches or surpasses it without such consumption. GiGPO assigns an additive step-level advantage to turns that visit the same state across rollouts, but in generative settings where states cannot be precisely identified, its gains remain marginal. IGPO, which shares the same underlying IG signal, shows limited improvement over GRPO, as its pooled normalization conflates heterogeneous positional contexts and its discounted accumulation introduces scale inconsistency across turn depths. A\textsuperscript{2}TGPO addresses these challenges by refining how the IG signal is normalized, accumulated, and consumed, enabling substantially larger gains from the same IG source.

  \begin{figure}[t]
      \centering
      \includegraphics[width=1.0\linewidth]{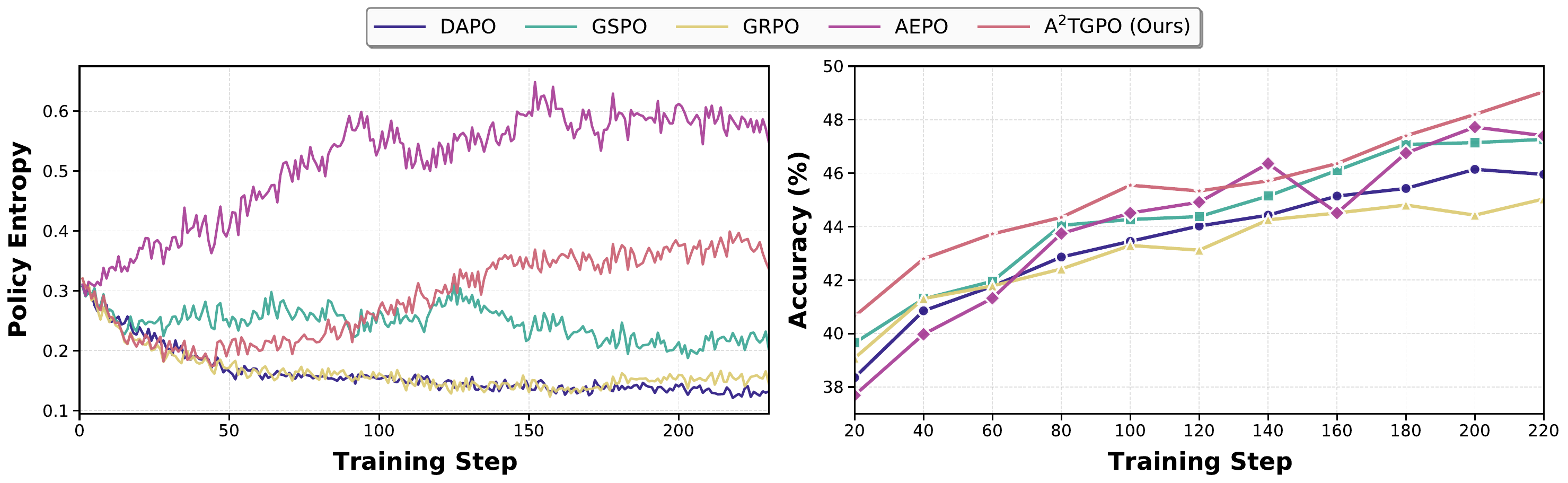}
      \vspace{-4.5mm}
      \caption{\textit{Left}: The entropy comparison during training on multi-hop benchmark. \textit{Right}: Performance comparison between classic baselines on HotpotQA dataset. Both are based on Qwen3-4B.}
      \label{fig:training_dynamics}
      \vspace{-4.5mm}
  \end{figure}

\subsection{Analysis}

\textbf{Ablation Study.} Table~\ref{tab:ablation_qwen3_4b_multihop} reports an additive ablation on Qwen3-4B multi-hop benchmarks, starting from IGPO and introducing our three components one at a time with GRPO as a non-IG reference. Plain IGPO gets limited gains as its pooled normalization conflates heterogeneous positional contexts, further distorting the IG signals. Turn-group normalization resolves the cross-position incomparability and already surpasses GRPO; variance-rescaled discounted accumulation further propagates credit backward while keeping advantage magnitudes comparable across turn depths; IG-adaptive turn-level clipping completes A\textsuperscript{2}TGPO with the largest gain on the longer-horizon MuSiQue and Bamboogle, where modulating the clipping range by per-turn informativeness proves most beneficial. Each component contributes a distinct, additive gain and finally achieves the best overall performance.

\textbf{Training Dynamics Analysis.} Figure~\ref{fig:training_dynamics} compares training dynamics of A\textsuperscript{2}TGPO against representative RLVR methods and a promising baseline AEPO. On the left panel, classic RLVR methods suffer rapid entropy collapse, suppressing exploration early in training, while AEPO exhibits the opposite trend with entropy climbing steadily. A\textsuperscript{2}TGPO maintains a balanced entropy plateau throughout training, preserving stronger exploration than RLVR methods while remaining more stable than AEPO. The right panel confirms that this balanced regime translates into consistently higher and more stable validation accuracy, with A\textsuperscript{2}TGPO leading all baselines throughout training.

\begin{figure}[t]
  \centering
  \includegraphics[width=0.99\linewidth]{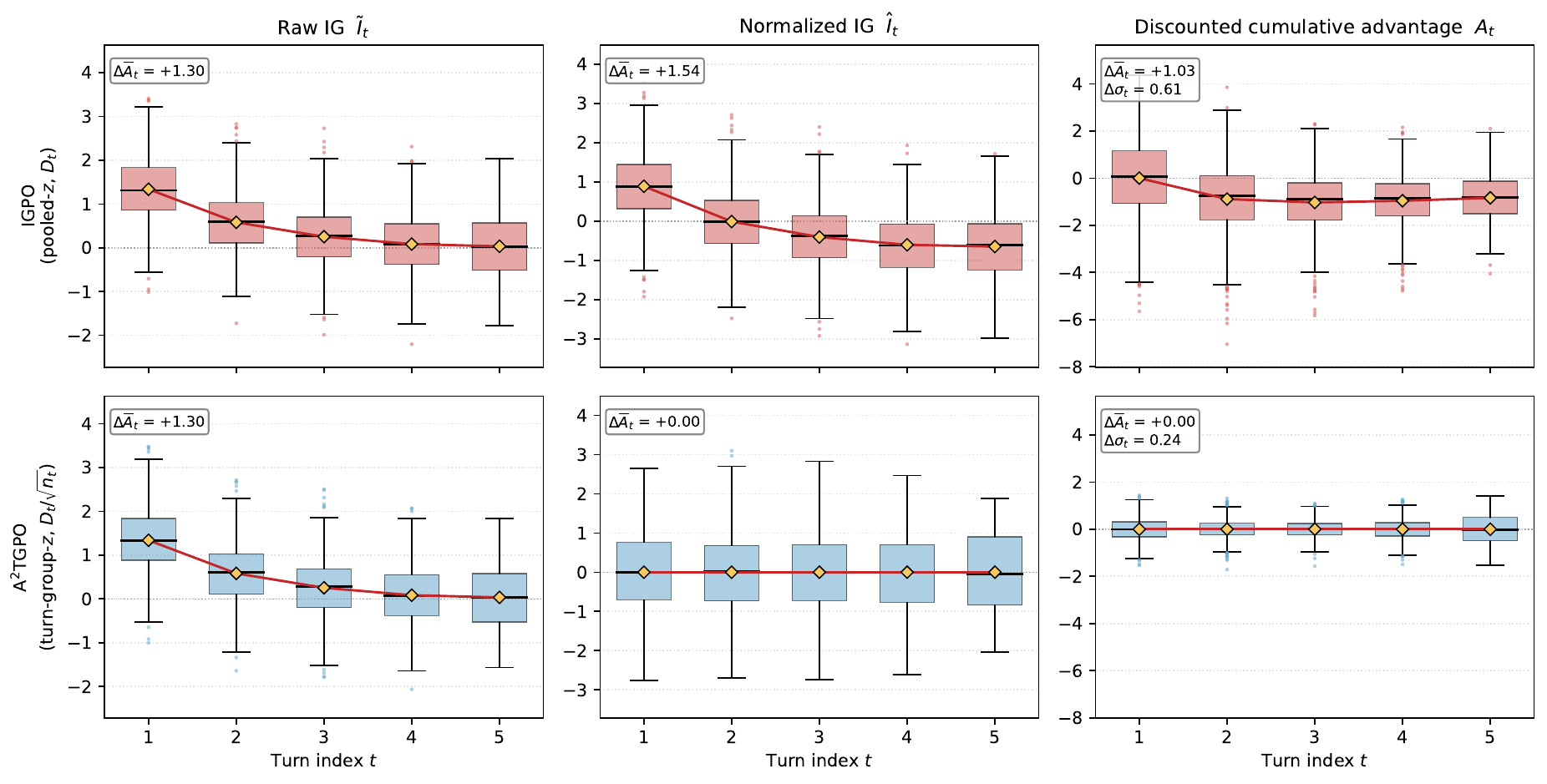}
  \vspace{-3.5mm}
  \caption{Within-step per-turn advantage distribution on multi-hop benchmarks based on Qwen3-4B. Three columns show the progression from raw IG, through normalized IG and to discounted cumulative advantage. Top row: \textbf{IGPO} v.s. Bottom row: \textbf{A\textsuperscript{2}TGPO (Ours)}.}
  \label{fig:pertu-adv-reconstruct}
  \vspace{-4.5mm}
\end{figure}

\begin{figure}[t]
  \centering
  \includegraphics[width=0.99\linewidth]{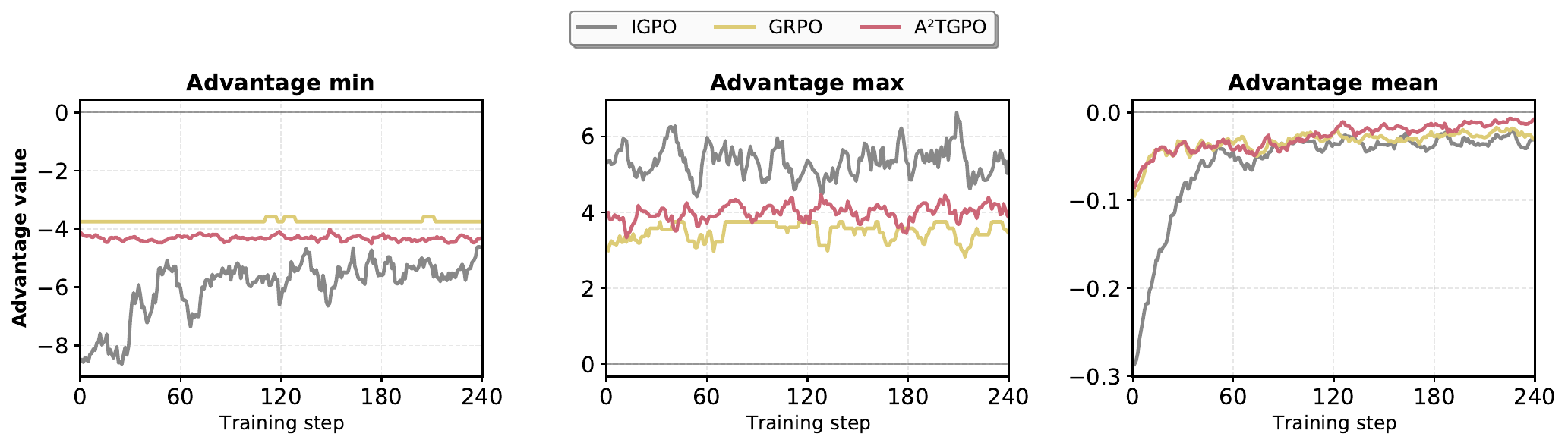}
  \vspace{-3.5mm}
  \caption{Advantage envelope dynamics over 240 training steps on multi-hop benchmarks based on Qwen3-4B. Three panels track the per-step advantage \textbf{minimum (left)}, \textbf{maximum (middle)}, and \textbf{mean (right)} for IGPO, GRPO and A\textsuperscript{2}TGPO throughout training.}
  \label{fig:adv-envelope-3way}
  \vspace{-5.7mm}
\end{figure}

\textbf{Training Advantage Distribution Analysis.} To intuitively illustrate the mechanism of A\textsuperscript{2}TGPO and its difference from IGPO, we present two complementary views. Figure~\ref{fig:pertu-adv-reconstruct} visualizes how the per-turn advantage mean distribution varies across three processing stages: original raw IG, normalized IG, and discounted cumulative advantage, which compares IGPO and A\textsuperscript{2}TGPO on the same training batch at step~240. The original raw IG column (left) is approximately identical for both methods, ensuring a fair comparison basis. In the normalization stage (middle), IGPO's pooled $z$-normalization distorts the relative standing of individual turns. Since the first tool call naturally tends to obtain a larger information gain relative to the no-tool baseline due to the inherent informativeness of initial tool call, pooled normalization mistakes this positional characteristic for genuine superiority and distorts the credit assigned to subsequent turns. A\textsuperscript{2}TGPO's turn-group normalization centers each position independently, eliminating this location bias. In the advantage stage (right), IGPO's cumulative aggregation amplifies the scale heterogeneity, producing widely spread boxes that span most of the $y$-axis range with large outliers; A\textsuperscript{2}TGPO's variance rescaling compresses the distribution into a narrow, consistent band across all turn positions, making each turn contribute comparable gradient signals to the optimizer. Figure~\ref{fig:adv-envelope-3way} tracks the advantage minimum, maximum, and mean over the full 240-step training in IGPO, GRPO and A\textsuperscript{2}TGPO. IGPO exhibits substantially wider min/max swings and a slower mean convergence toward zero, reflecting the location and scale drift diagnosed above. A\textsuperscript{2}TGPO maintains an advantage envelope comparable to GRPO's bounded range while retaining the process credit signals that GRPO lacks, achieving both distributional stability and fine-grained per-turn evaluation simultaneously.

\section{Conclusion}

\label{sec:conclusion}                 
   
We presented \textbf{A\textsuperscript{2}TGPO}, which re-designs the utilization of information gain(IG) as an intrinsic per-turn process signal in agentic reinforcement learning. By grouping IG statistics per (prompt, turn-index), rescaling discounted cumulative advantages to equalize cross-depth magnitudes, and adapting the clipping range to per-turn informativeness, A\textsuperscript{2}TGPO delivers stable, fine-grained process credit assignment without any external evaluator or additional rollout. On seven single-hop and multi-hop QA benchmarks across three backbones, A\textsuperscript{2}TGPO consistently outperforms both general-purpose RLVR methods and specialized agentic baselines, while maintaining a balanced entropy regime throughout training. Extending the framework to broader tool suites, longer-horizon agentic tasks is left to future work.

{
\small

\bibliographystyle{unsrtnat} 
\bibliography{custom}
}

\newpage
\appendix

\section{Algorithm Workflow}
\label{sec:algorithm}

Algorithm~\ref{alg:atgpo} summarizes the complete A\textsuperscript{2}TGPO training procedure. Each iteration consists of four stages: (1)~multi-turn rollout generation with interleaved tool calls, (2)~per-turn information gain computation via forward logits prediction over the ground-truth answer, (3)~turn-group normalization, discounted accumulation with variance-rescaling, and IG-based adaptive turn-level clipping, and (4)~policy update with the turn-level clipped policy objective.

\begin{algorithm}[h] 
\caption{A\textsuperscript{2}TGPO Workflow}
\label{alg:atgpo}
\begin{algorithmic}[1]
\REQUIRE Policy $\pi_\theta$; environment $\mathcal{E}$; training prompts $\{q\}$; group size $G$; discount $\gamma$; clip bounds $(\epsilon_{\mathrm{low}}, \epsilon_{\mathrm{high}})$; adaptation strength $\beta$
\FOR{each training iteration}
  \STATE \textbf{// Stage 1: Multi-turn Rollout Generation}
  \FOR{each prompt $q$}
    \STATE Sample $G$ trajectories $\{\tau_i\}_{i=1}^{G}$ from $\pi_{\theta_{\mathrm{old}}}$, where $\tau_i = (y_{i,1}, o_{i,1}, \dots, y_{i,T_i})$
  \ENDFOR
  \STATE \textbf{// Stage 2: Information Gain Computation (Eq.~\ref{eq:ig-def})}
  \FOR{each trajectory $\tau_i$, each process turn $t \in \{1,\dots,T_i{-}1\}$}
    \STATE Compute $\mathrm{ig}_{i,t} = \pi_\theta(a \mid q, \tau_{i,\leq t}) - \pi_\theta(a \mid q, \tau_{i,\leq t-1})$ \hfill $\triangleright$ single forward logits prediction
  \ENDFOR
  \STATE \textbf{// Stage 3: Advantage Construction}
  \FOR{each prompt $q$, each turn index $t$}
    \STATE Collect turn-group $\mathcal{G}_{q,t} = \{\mathrm{ig}_{i,t} : T_i \geq t\}$ \hfill $\triangleright$ \textit{Turn-Group Normalization} (Eq.~\ref{eq:tg-normed})
    \IF{$|\mathcal{G}_{q,t}| \geq 2$}
      \STATE $\widehat{\mathrm{ig}}_{i,t} \leftarrow (\mathrm{ig}_{i,t} - \mathrm{mean}(\mathcal{G}_{q,t})) \,/\, \mathrm{std}(\mathcal{G}_{q,t})$
    \ELSE
      \STATE $\widehat{\mathrm{ig}}_{i,t} \leftarrow 0$ \hfill $\triangleright$ no peer for comparison
    \ENDIF
  \ENDFOR
  \FOR{each trajectory $\tau_i$}
    \STATE Compute outcome advantage $\widehat{R}_i$ via per-prompt GRPO normalization
    \FOR{$t = T_i{-}1$ \textbf{down to} $1$}
      \STATE $D_{i,t} \leftarrow \sum_{k=t}^{T_i-1} \gamma^{k-t}\,\widehat{\mathrm{ig}}_{i,k}$ \hfill $\triangleright$ \textit{Discounted Cumulative IG} (Eq.~\ref{eq:cumulative-ig})
      \STATE $\widehat{A}_{i,t} \leftarrow D_{i,t}\,/\,\sqrt{T_i - t} + \widehat{R}_i$ \hfill $\triangleright$ \textit{Variance Rescaling} (Eq.~\ref{eq:atgpo-adv})
    \ENDFOR
    \STATE $\widehat{A}_{i,T_i} \leftarrow \widehat{R}_i$ \hfill $\triangleright$ final answer turn: outcome only
  \ENDFOR
  \STATE \textbf{// Stage 4: Policy Update with Adaptive Turn-level Clipping}
  \FOR{each trajectory $\tau_i$, each turn $t$}
    \STATE $s_{i,t} \leftarrow \exp\!\big(\frac{1}{|y_{i,t}|}\sum_{k}\log\frac{\pi_\theta(y_{i,t,k}\mid\cdot)}{\pi_{\theta_{\mathrm{old}}}(y_{i,t,k}\mid\cdot)}\big)$ \hfill $\triangleright$ \textit{Turn-level IS ratio} (Eq.~\ref{eq:turn-is})
    \STATE $c_{i,t} \leftarrow 1 + \beta\,(2\,\sigma(\widehat{\mathrm{ig}}_{i,t}) - 1)$ \hfill $\triangleright$ \textit{IG-based adaptive clipping range scale} (Eq.~\ref{eq:clip-scale})
  \ENDFOR
  \STATE Update $\theta$ by minimizing $\mathcal{L}_{\mathrm{A^2TGPO}}(\theta)$ (Eq.~\ref{eq:atgpo-loss}) with effective bounds $(c_{i,t}\epsilon_{\mathrm{low}},\; c_{i,t}\epsilon_{\mathrm{high}})$
\ENDFOR
\end{algorithmic}
\end{algorithm}

\section{Implementation Details}
\label{sec:implementation}

\subsection{Reward Design}

\label{sec:reward_design}                                                                                                                                                                        
The outcome reward used in our agentic RL pipeline is set as a binary and rule-based signal that jointly accounts for answer correctness and response formatting. The correctness component follows the reward specification of Search-R1~\cite{jin2025search}, which takes the EM score as the primary metric, and is complemented by an explicit structural constraint on the generated output.

\textbf{Exact Match Reward} Let $\hat{y}$ denote the final answer extracted from the agent's completed trajectory, and let $y^{*}$ denote the associated ground-truth answer. The \textbf{Exact Match (EM)} score serves as the primary metric for answer correctness:
\begin{equation}
  r_{\mathrm{EM}}\left(\hat{y}, y^{*}\right)=\left\{\begin{array}{ll}1, & \text { if } \hat{y}=y^{*} \\0, & \text { otherwise }\end{array}\right.
\end{equation}

The strict binary form removes the ambiguity inherent to partial-credit signals and pushes the policy toward fully correct predictions rather than hedged or partially overlapping outputs, which provides a higher upper bound for agentic RL.
   
\textbf{Format Constraint} Beyond answer correctness, a structural validation requirement is imposed on every trajectory. Specifically, each response is required to contain both a reasoning trace and a final-answer segment, wrapped respectively by the tags \textit{<think>...</think>} and \textit{<answer>...</answer>}. Within the answer block, the string consumed by the EM checker is further required to be enclosed in \textit{\textbackslash{}boxed\{\}}. The corresponding validity metric is defined as
\begin{equation}
  \mathbb{I}_{\text {format }}=\left\{\begin{array}{ll}1, & \text { if both tags are present } \\0, & \text { otherwise }\end{array}\right.
\end{equation}

Any response violating the required schema receives no credit regardless of the underlying answer, giving the trajectory parser a reliable way for tool-call parsing and final-answer extraction while discouraging drift toward free-form outputs that are hard to verify downstream.

\textbf{Final Reward Definition} Combining the two components, the overall reward $r$ is composed of the exact-match score and the format indicator as follows:
\begin{equation}
  r=\left\{\begin{array}{ll}r_{\mathrm{EM}}(\hat{y},y^*), & \text { if $\mathbb{I}_{\text {format }}=1$ } \\-1, & \text { otherwise }\end{array}\right.
\end{equation}

Under this composite rule, a trajectory earns the maximal reward of $1$ only when it simultaneously complies with the required schema and delivers an answer that exactly matches the reference; any format deviation is explicitly penalised.

\subsection{Prompt Template}

\begin{figure*}[t]
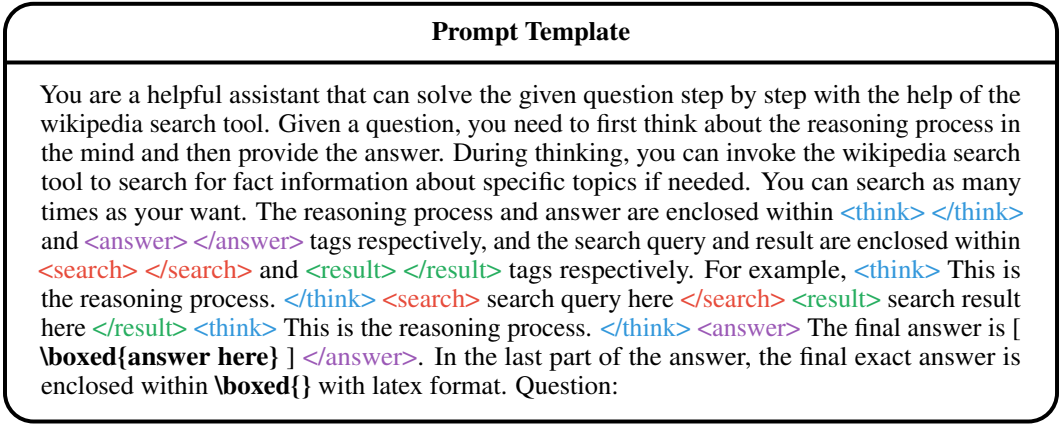

\centering
\begin{tcolorbox}[
    colback=white,
    colframe=black,
    arc=10pt,
    boxrule=1.5pt,
    left=10pt,
    right=10pt,
    top=5pt,
    bottom=5pt,
    toptitle=3pt,
    bottomtitle=3pt,
    title={\centering\textcolor{black}{\normalsize\bfseries Prompt Template}},
    colbacktitle=white,
    coltitle=white,
    fonttitle=\bfseries\Large
]

You are a helpful assistant that can solve the given question step by step with the help of the wikipedia search tool. Given a question, you need to first think about the reasoning process in the mind and then provide the answer. During thinking, you can invoke the wikipedia search tool to search for fact information about specific topics if needed. You can search as many times as your want. The reasoning process and answer are enclosed within \textcolor{myblue}{<think> </think>} and \textcolor{mypurple}{<answer> </answer>} tags respectively, and the search query and result are enclosed within \textcolor{myred}{<search> </search>} and \textcolor{mygreen}{<result> </result>} tags respectively. For example, \textcolor{myblue}{<think>} This is the reasoning process. \textcolor{myblue}{</think>} \textcolor{myred}{<search>} search query here \textcolor{myred}{</search>} \textcolor{mygreen}{<result>} search result here \textcolor{mygreen}{</result>} \textcolor{myblue}{<think>} This is the reasoning process. \textcolor{myblue}{</think>} \textcolor{mypurple}{<answer>} The final answer is [ \textbf{\textbackslash{}boxed\{answer here\}} ] \textcolor{mypurple}{</answer>}. In the last part of the answer, the final exact answer is enclosed within \textbf{\textbackslash{}boxed\{\}} with latex format. Question:

\end{tcolorbox}
\caption{The prompt template in our experiment setting.}
\label{prompt template}
\end{figure*}

Figure~\ref{prompt template} illustrates the instruction template adopted in this work, which adapts the tag-based response schema used in \cite{jin2025search,dong2025agenticaepo,zong2026at2poagenticturnbasedpolicy} to the search-augmented agentic reasoning setting. The schema partitions each rollout into four semantically distinct regions, with every region delimited by a dedicated pair of tags. Intermediate deliberation is verbalised inside \textcolor{myblue}{\textit{<think></think>}}, making the reasoning chain explicit and separable from tool interactions. Whenever the policy decides that external evidence is required, a retrieval action is issued by emitting a query wrapped in \textcolor{myred}{\textit{<search></search>}}; the evidence returned by the tool is subsequently injected back into the context inside \textcolor{mygreen}{\textit{<result></result>}}, so that retrieved passages are clearly marked as environment feedback rather than model-generated content. Once sufficient evidence has been accumulated, the final prediction is emitted inside \textcolor{mypurple}{\textit{<answer></answer>}}, and the short span wrapped by \textbf{\textbackslash{}boxed\{\}} is treated as the canonical prediction from which Exact Match is extracted. 


\subsection{Datasets}

\label{sec:datasets}
  Our main experiments are conducted on two categories of open-domain question answering benchmarks, chosen to evaluate the ability of search-augmented agentic reasoning.            
               
  \noindent \textbf{Multi-Hop QA.} This category benchmark targets the multi-turn tool-use and compositional reasoning abilities, which consists of datasets in which a correct answer cannot be obtained from any single retrieved passage. \textbf{HotpotQA}~\cite{yang2018hotpotqa} is a large-scale Wikipedia-derived benchmark annotated with supporting-fact supervision, and remains one of the most commonly used testbeds for explainable multi-hop question answering. \textbf{2WikiMultiHopQA}~\cite{ho2020constructing} couples Wikipedia passages with Wikidata triples, yielding questions whose answers rely on explicit multi-hop entity dependencies rather than surface-level lexical cues. \textbf{MuSiQue}~\cite{trivedi2022musique} contains roughly 25k questions of 2--4 hops that are synthesised through controlled composition of single-hop primitives, making it particularly suited to probing fine-grained reasoning depth.
  \textbf{Bamboogle}~\cite{press2023measuring} provides a small yet adversarial collection of compositional queries and is retained as a robustness probe for evaluating the stability of agentic RL policies on compositional patterns.

  \noindent \textbf{Single-Hop QA.} The second category is used to verify the gains on the single-step retrieval regime.
   \textbf{Natural Questions (NQ)}~\cite{kwiatkowski2019natural} aggregates real user queries answered from Wikipedia and is a standard yardstick for retrieval-augmented generation.
  \textbf{TriviaQA}~\cite{joshi2017triviaqa} introduces substantial lexical and syntactic divergence between questions and their supporting evidence, stressing robustness to surface variation.
  \textbf{PopQA}~\cite{mallen2023llm_memorization} is an entity-centric open-domain benchmark curated to disentangle the contribution of external retrieval from parametric memorisation, which makes it a natural diagnostic for whether the policy genuinely exploits the search tool rather than relying on memorised facts.

\subsection{A\textsuperscript{2}TGPO Settings}

For implementation details of our A\textsuperscript{2}TGPO, we use a training batch size of $64$, a mini-batch size of $8$, and a maximum response length of $6192$. During rollout, we use a rollout size of $16$, the same as other baselines, with the maximum tool usage set to $6$. The clipping thresholds for the A\textsuperscript{2}TGPO objective are set to $3\text{e-}3$ and $4\text{e-}3$ (the same as GSPO). The discount factor $\gamma$ for accumulated advantage computation is set to $1.0$, as the $\sqrt{n_t}$ rescaling ensures consistent advantage magnitudes across turn depths. Additionally, we set $\beta = 0.3$, bounding the adaptive clipping scale within $(0.7, 1.3)$.

\subsection{Baseline Settings}

\label{sec:baseline_settings}
All baselines experiments are conducted via respective RL recipes without any additional SFT phase. The hyperparameter configuration shared across all RL-based baselines in the main experiment is reported in Table~\ref{tab:hyperparameters}, and the method-specific deviations from this default are described below. The listed baselines are selected in our experiments, which together cover the dominant design choices currently explored for agentic RL training. We compute the average EM accuracy across all evaluation samples. We select and report results from the checkpoint with the highest average score.                                                                                                          
  \begin{itemize}[leftmargin=*]
      \item \textbf{ReAct}~\cite{yao2023react}: a training-free prompting paradigm that interleaves intermediate reasoning with dynamic tool-calls. It serves as a non-RL reference point in which the policy parameters remain untouched.
      \item \textbf{GRPO}~\cite{guo2025deepseek}: an on-policy, critic-free RL objective that replaces the PPO value network by contrasting $n$ rollouts drawn from the same prompt, and uses the within-group standardised return as the advantage estimate. The KL penalty coefficient is set to $0.001$ and the symmetric clip ratio to $0.2$.                   
      \item \textbf{DAPO}~\cite{yu2025dapo}: a GRPO variant that decouples the low/high clipping thresholds and applies dynamic sampling to improve update stability. We adopted the default configuration with \texttt{clip\_ratio\_low}$=0.2$ and \texttt{clip\_ratio\_high}$=0.28$ for a wider upper clip range. The \texttt{overlong\_buffer} of capacity $2000$ with penalty factor $1.0$ and dynamic sampling is enabled. Following the official recipe, the generation batch size is set to three times the training batch size.
      \item \textbf{GSPO}~\cite{zheng2025groupsequencepolicyoptimization}: a sequence-level reformulation of GRPO in which the importance ratio and clipping are both computed on the full-response likelihood, sacrificing token-level granularity for sequence-level stability. Similar to DAPO, asymmetric clipping is used with \texttt{clip\_ratio\_low}$=3{\times}10^{-4}$ and \texttt{clip\_ratio\_high}$=4{\times}10^{-4}$, matching the tight thresholds required at the sequence level.
      \item \textbf{AEPO}~\cite{dong2025agenticaepo}: a state-of-the-art agentic RL method that combines entropy-balanced rollout scheduling with entropy-aware clipping to suppress over-branching and gradient collapse in tree-structured rollouts. The configuration used here sets \texttt{initial\_rollouts}$=8$, \texttt{beam\_size}$=2$, \texttt{branch\_probability}$=0.5$, and \texttt{entropy\_weight}$=0.2$.
      \item \textbf{Tree-GRPO}~\cite{ji2025treesearchllmagent}: a tree-structured agentic RL framework that couples GRPO with explicit tree search to enable fine-grained credit assignment across branching trajectories. The configuration from the original paper is reproduced without modification.
      \item \textbf{GIGPO}~\cite{feng2025group}: a method which introduces a hierarchical grouping scheme that identifies actions taken under the same state across different trajectories and estimates group-relative advantages at the step level. We follow its original hyperparameter settings and adopt the same state-identification approach according to the thresholds of similarity described in previous work.
      \item \textbf{IGPO}~\cite{wang2026information}: a promising method that uses the per-turn information gain toward the ground-truth answer as an intrinsic reward signal and derives turn-level advantages through pooled normalization and discounted accumulation. We follow its original hyperparameter settings without modification.
  \end{itemize}

  \begin{table}[t]
  \centering
  \small
  \setlength\tabcolsep{8pt}
  \renewcommand\arraystretch{1.2}
  \begin{tabular}{lcccccccc}
  \hline
  Config & Value  \\
  \hline
  optimizer & AdamW  \\
  learning rate & 1e-6  \\
  clip\_ratio & 0.2 \\
  training batch size & 64 \\
  PPO mini batch size & 8 \\
  rollout\_n & 16 \\
  max prompt length & 2000 \\
  max response length & 6192  \\
  max tool-call turns & 6 \\
  reward metrics & EM  \\
  retriever & local wiki \\
  top-K retrieval passages & 3 \\
  \hline
  \end{tabular}
  \caption{Shared hyperparameters used by the baselines in our experiment.}
  \label{tab:hyperparameters}
  \end{table}

\subsection{Search Tool Environment}
\label{sec:search_tool_settings}

The search tool used throughout training and evaluation mirrors the configuration of Search-R1~\cite{jin2025search}: a Wikipedia snapshot serves as the retrieval corpus, and \texttt{e5-base-v2}~\cite{wang2022text} is employed as the dense retriever. The underlying knowledge base comprises approximately $21$M Wikipedia entries, which provides broad factual coverage for both single-hop and multi-hop queries. At every turn in which the agentic policy emits a retrieval action, the engine scores all candidate passages against the issued query and returns the top-$k$ most relevant entries, which are then injected back into the context as tool feedback for subsequent reasoning. 

\subsection{Hardware and Artifacts}
\label{sec:hardware_and_artifacts}

All training and evaluation runs are executed on a single node equipped with $8\times$ NVIDIA H20 GPUs. Three publicly released checkpoints are adopted as backbone policies, namely Qwen3-4B, Qwen3-8B, and Qwen2.5-7B~\cite{yang2025qwen3technicalreport, qwen2025qwen25technicalreport}, which are selected for their strong basic reasoning ability and their demonstrated compatibility with agentic post-training. The training stack is implemented on top of the VeRL framework~\cite{sheng2024hybridflow}, a mature hybrid-controller RL infrastructure whose modular rollout interface integrates cleanly with the multi-turn, tool-interactive rollout schedule required by the proposed method.

\subsection{Limitations and Future Work}
\label{sec:limitations}

A\textsuperscript{2}TGPO relies on a ground-truth answer to compute the information gain signal, which limits its direct applicability to tasks with verifiable outcomes (e.g., open-ended creative generation or subjective evaluation). Exploring alternative intrinsic signals such as self-consistency across rollouts or uncertainty reduction in the policy's internal representations is a promising direction for broadening the method's scope. Additionally, while the IG forward pass adds only modest overhead in our current setting, designing more efficient IG computation strategies (e.g., amortized estimation or cached incremental updates) would further reduce cost and enable scaling to longer-horizon agentic tasks with dozens of interaction turns.

\section{Extended Experimental Analysis}

\subsection{Computational Overhead}
\label{sec:overhead}

\begin{figure}[t]
  \centering
  \begin{minipage}[t]{0.56\linewidth}
    \centering
    \includegraphics[width=\linewidth]{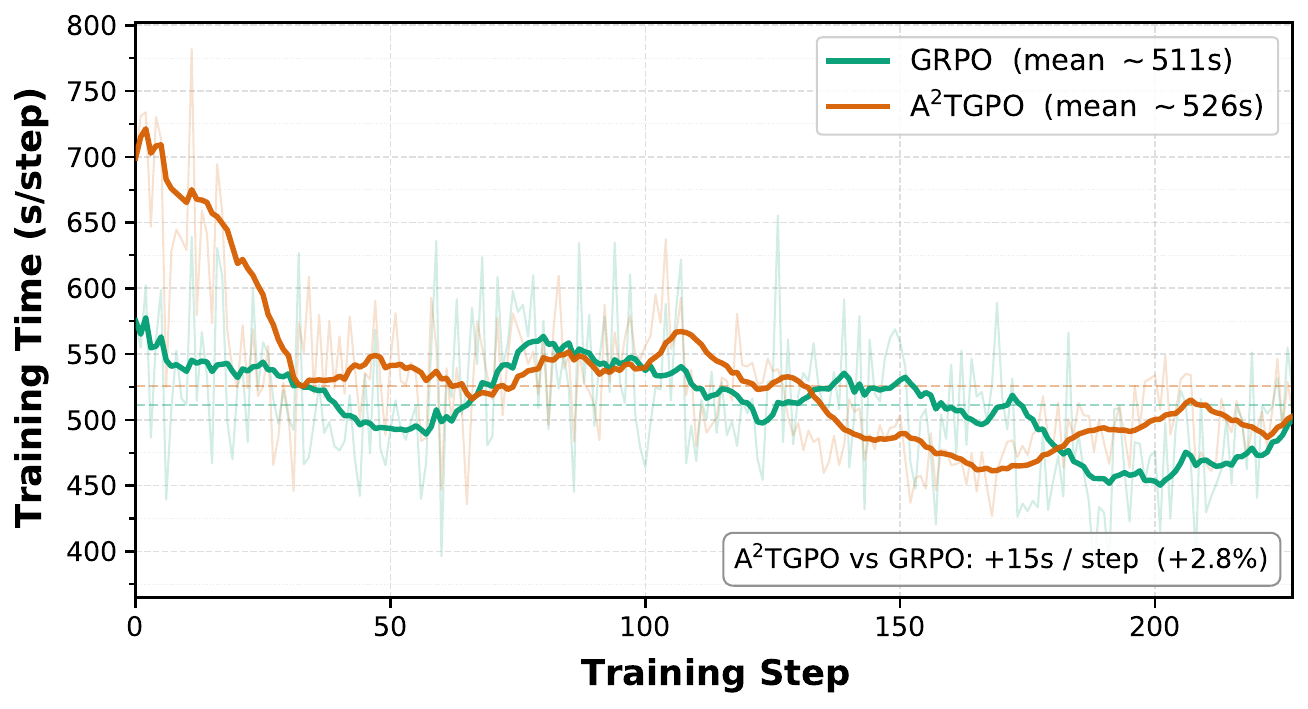}
  \end{minipage}%
  \hfill
  \begin{minipage}[t]{0.38\linewidth}
    \centering
    \includegraphics[width=\linewidth]{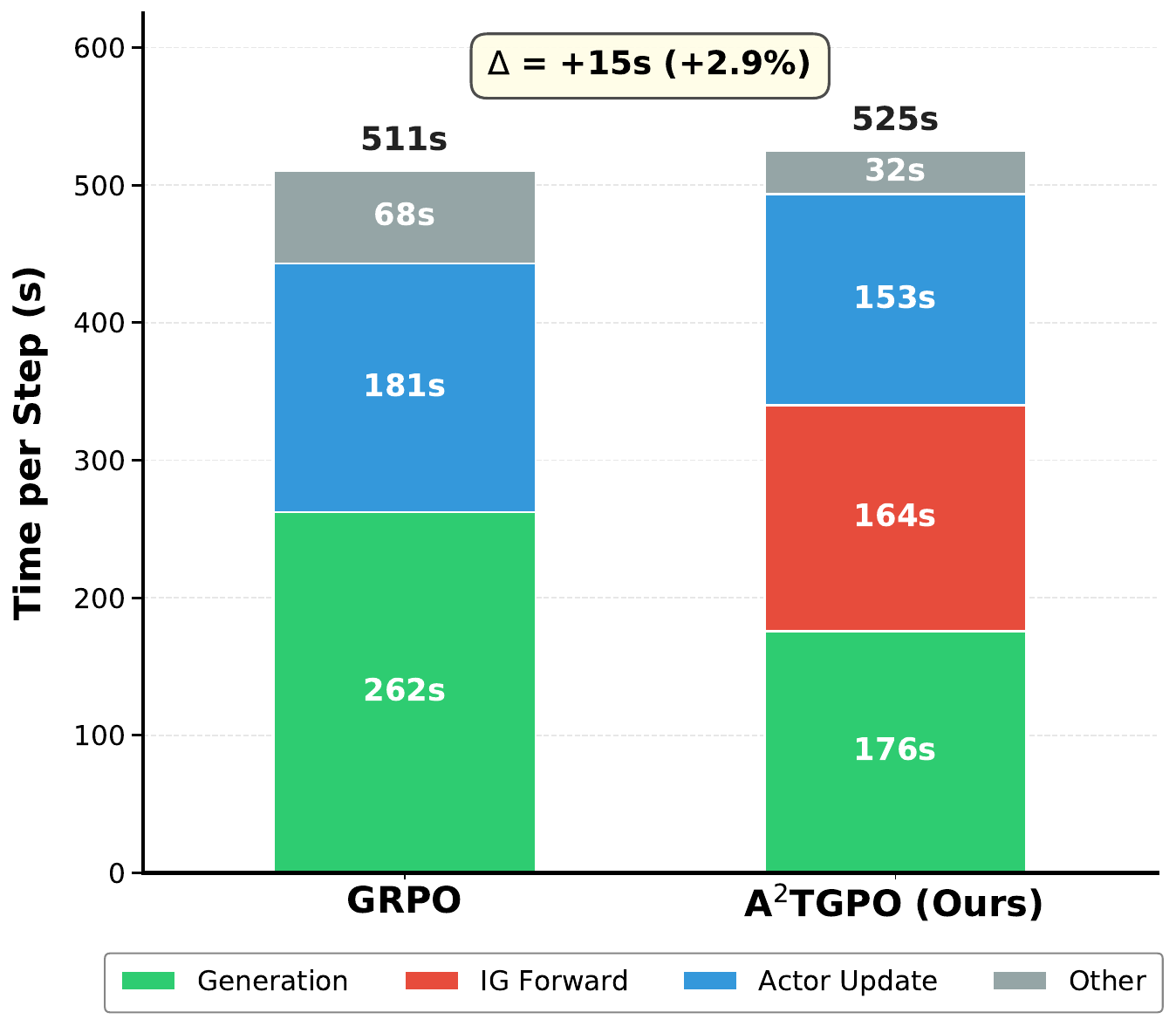}
  \end{minipage}
  \caption{\textit{Left}: Per-step training time on Qwen3-4B multi-hop QA under rollout budget $n=16$. \textit{Right}: Average per-step time breakdown over 240 training steps. The IG forward pass is A\textsuperscript{2}TGPO's sole additional component ($+164$\,s), whose cost is largely offset by faster generation ($-86$\,s), resulting in a net overhead of only $+15$\,s ($+2.9\%$).}
  \label{fig:overhead}
  \vspace{-4.5mm}
\end{figure}

\begin{figure}[t]
  \centering
  \includegraphics[width=1.0\linewidth]{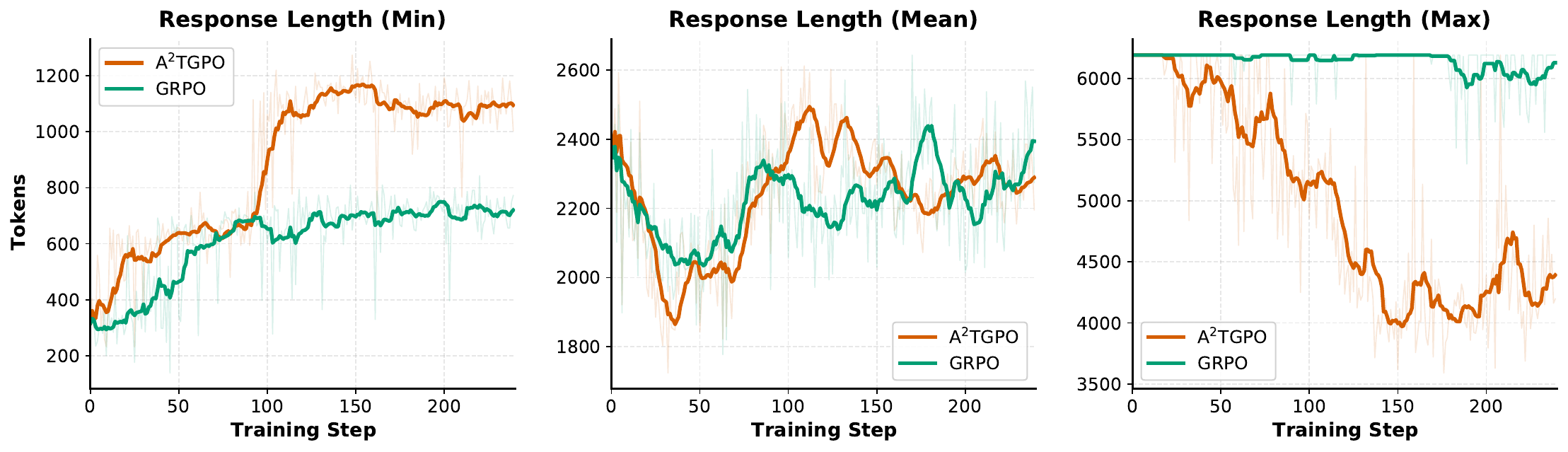}
  \vspace{-4.5mm}
  \caption{Response length statistics (min, mean, max) over 240 training steps. A\textsuperscript{2}TGPO produces a tighter length distribution: higher minimum, comparable mean, and substantially lower maximum.}
  \label{fig:response-length}
  \vspace{-4.7mm}
\end{figure}

Figure~\ref{fig:overhead} (left) compares per-step wall-clock time under a matched rollout budget ($n=16$). A\textsuperscript{2}TGPO averages ${\sim}525$\,s/step versus GRPO's ${\sim}511$\,s/step, a net overhead of only $+2.9\%$. The sole additional cost that the IG forward pass (${\sim}164$\,s) which follows the efficient single-pass implementation of \citet{wang2026information}. Note that it is largely offset by faster generation ($-86$\,s) as discussed below. Figure~\ref{fig:overhead} (right) shows the average time breakdown. The elevated training time of A\textsuperscript{2}TGPO during the first ${\sim}50$ steps (Figure~\ref{fig:overhead}, left) stems from its early-stage exploration pattern. As Figure~\ref{fig:response-length} shows, A\textsuperscript{2}TGPO initially produces a higher minimum response length, indicating that all rollouts engage in substantial exploration rather than terminating prematurely. After this transient phase, the per-turn credit signal teaches the policy to terminate efficiently once sufficient evidence is gathered: the maximum response length drops from ${\sim}6000$ to ${\sim}4000$ tokens, whereas GRPO consistently produces rollouts hitting the token limit (${\sim}6192$) throughout training. Since parallel generation is bottlenecked by the \emph{longest} sequence in each batch, this yields progressively faster generation for A\textsuperscript{2}TGPO, resulting in a tighter and more balanced length distribution overall.

\subsection{Ablation Study on Single-hop Benchmarks}
\label{sec:ablation-singlehop}

\begin{table}[t]
    \centering
    \footnotesize
    \setlength{\tabcolsep}{5pt}
    \renewcommand{\arraystretch}{1.0}
    \caption{Ablation study on A\textsuperscript{2}TGPO components on single-hop benchmarks based on Qwen3-4B.}
    \label{tab:ablation_singlehop}
    \begin{tabular}{l|cccc}
    \toprule
    \textbf{Method} & NQ & TriviaQA & PopQA & Avg. \\
    \midrule
    GRPO                                   & 45.98 & 65.17 & 49.18 & 54.97 \\
    \midrule
    IGPO                                   & 46.05 & 65.25 & 49.38 & 55.12 \\
    \;+ TG-Norm                            & 46.55 & 65.75 & 49.81 & 55.58 \\
    \;+ TG-Norm + $D_t$-rescaling         & 46.77 & 65.95 & 49.98 & 55.77 \\
    \;+ TG-Norm + $D_t$-rescaling + Ada-Clipping \textbf{(A\textsuperscript{2}TGPO)} & \textbf{47.50} & \textbf{66.65} & \textbf{50.60} & \textbf{56.44} \\
    \bottomrule
    \end{tabular}
\end{table}

Table~\ref{tab:ablation_singlehop} extends the ablation to single-hop benchmarks. All three components remain effective. IG-adaptive turn-level clipping contributes the largest gain (${\sim}50\%$) as single-hop tasks involve fewer turns where modulating the per-turn update intensity becomes the primary factor. Turn-group normalization provides a moderate gain (${\sim}35\%$), while variance-rescaled accumulation contributes less (${\sim}15\%$) since single-hop trajectories typically contain $1$ to $2$ process turns, making it less important for cross-depth scale correction.

\subsection{Sensitivity to Adaptive Clipping Coefficient $\beta$}
\label{sec:beta-sensitivity}

\begin{figure}[t]
  \centering
  \includegraphics[width=\linewidth]{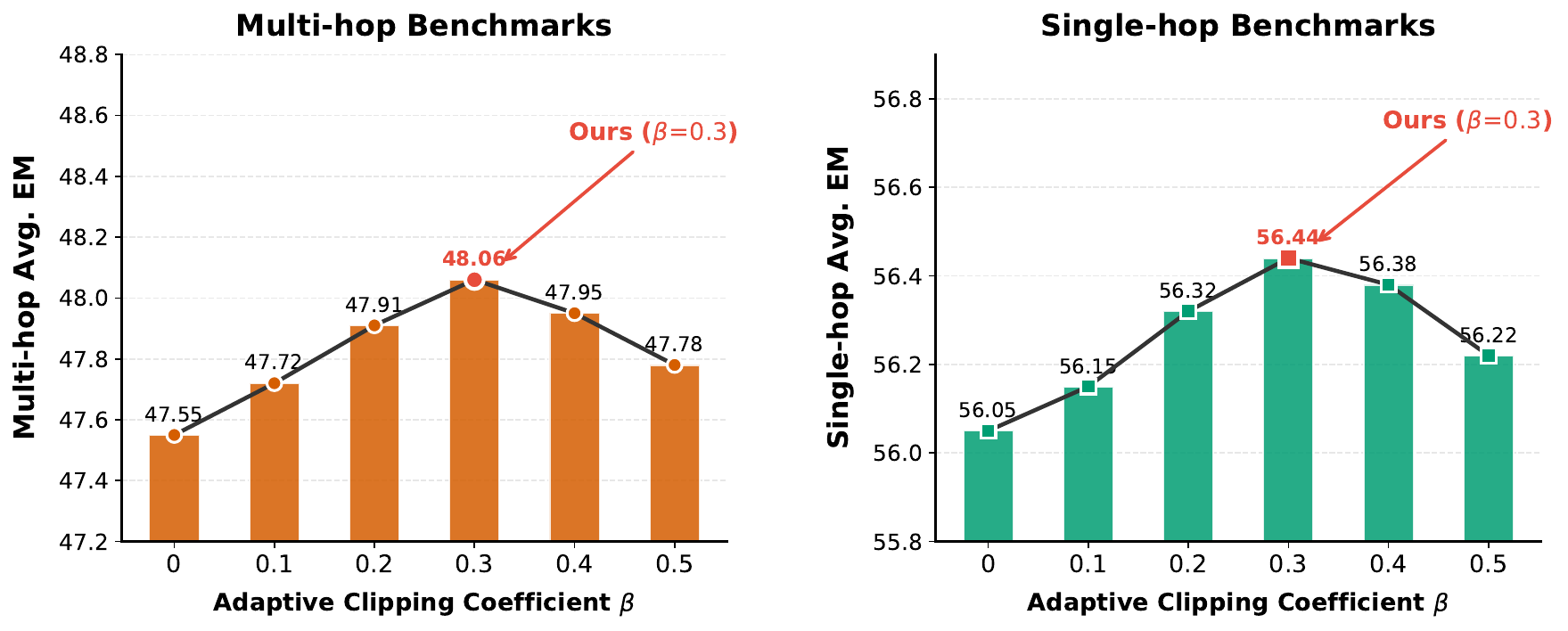}
  \vspace{-4.5mm}
  \caption{Sensitivity of A\textsuperscript{2}TGPO to the adaptive clipping coefficient $\beta$ (Eq.~\eqref{eq:clip-scale}). $\beta{=}0$ reduces to a fixed clipping range. Both benchmarks exhibit a clear trend peaking at $\beta{=}0.3$, and performance remains stable across $\beta \in [0.2, 0.4]$.}
  \label{fig:beta-sensitivity}
  \vspace{-4.7mm}
\end{figure}

The hyperparameter $\beta$ governs the adaptive clipping range $c_{i,t} \in [1{-}\beta,\; 1{+}\beta]$ (Eq.~\eqref{eq:clip-scale}). Figure~\ref{fig:beta-sensitivity} sweeps $\beta \in \{0, 0.1, \ldots, 0.5\}$. At $\beta{=}0$ (fixed clipping, no adaptive clipping), A\textsuperscript{2}TGPO already outperforms GRPO and IGPO, confirming the standalone effectiveness of turn-group normalization and variance rescaling. Performance improves monotonically up to $\beta{=}0.3$ (multi-hop 48.06, single-hop 56.44) and degrades only mildly beyond it, remaining within 0.15 points of the optimum across $\beta \in [0.2, 0.4]$. We fix $\beta{=}0.3$ in all experiments.

\subsection{Turn Distribution Analysis}
\label{sec:turn-distribution}

\begin{figure}[t]
  \centering
  \includegraphics[width=0.8\linewidth]{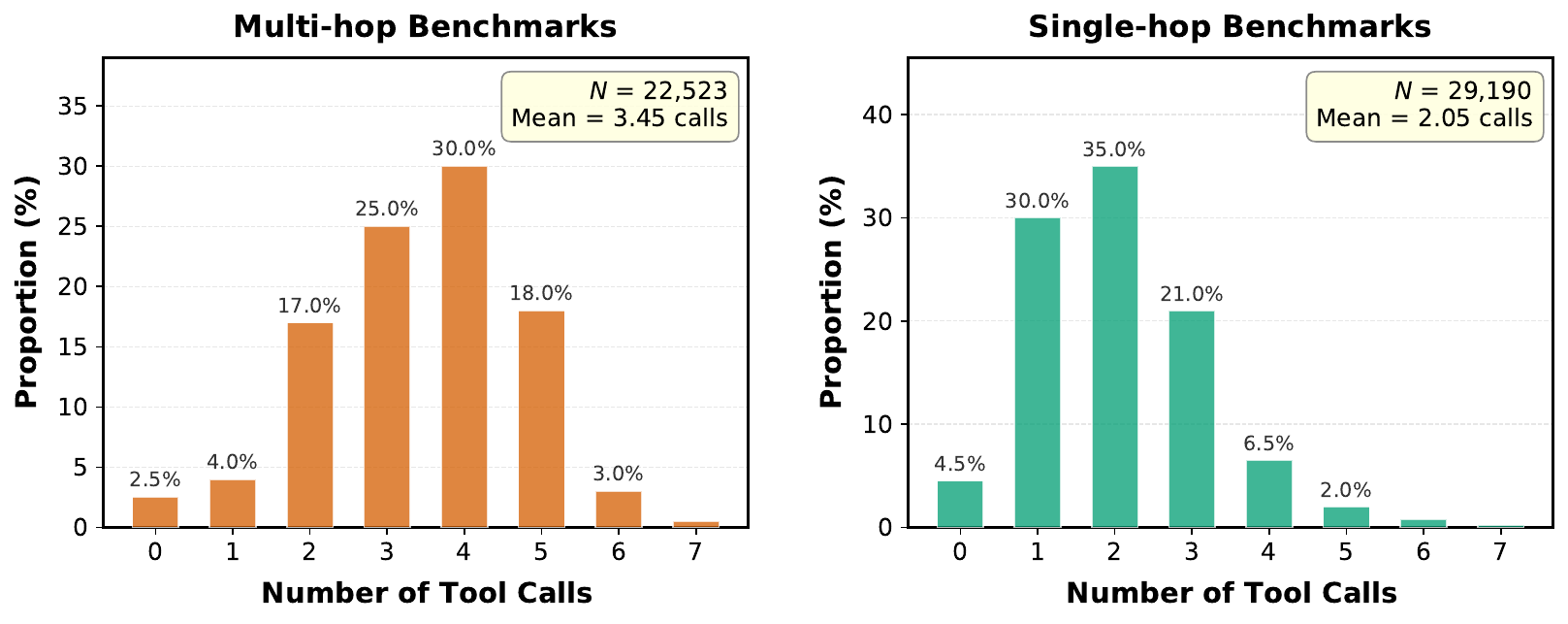}
  \vspace{-4.0mm}
  \caption{Distribution of the number of tool calls per rollout on multi-hop and single-hop benchmarks.}
  \label{fig:turn-distribution}
\end{figure}

Figure~\ref{fig:turn-distribution} visualizes the distribution of tool-call counts across rollouts during the validation stage at step 240. Multi-hop tasks elicit a mean of 3.45 calls per rollout, with the majority (72\%) falling in the 2--4 range, reflecting the inherent need for multiple retrieval steps to bridge reasoning chains. Single-hop tasks concentrate at fewer calls (mean 2.05), with 65\% of rollouts using 1--2 tool calls. Notably, both settings exhibit a non-trivial proportion of zero-call rollouts (2.5\% and 4.5\%), where the model judges the answer to be directly derivable without external retrieval. The broad, variable-length distributions confirm that A\textsuperscript{2}TGPO operates in a genuinely heterogeneous turn-length regime, showing that the turn-group normalization (Section~\ref{sec:method-tgnorm}) is critical for unbiased credit assignment.

\subsection{Context Similarity Analysis}
\label{sec:context-similarity}

A key premise of turn-group normalization is that rollouts at the same turn position share similar contexts. Figure~\ref{fig:context-similarity} verifies this by computing pairwise Jaccard similarity of per-turn content (thinking, query, and tool-response results) across rollouts of the same prompt. At step~1 (initial policy), rollouts at the same position already exhibit moderate similarity ($\sim$0.6--0.7) that decreases with depth, confirming that turns in the same position naturally share similar and comparable contexts, especially before trajectories branch substantially. After training (step~250), intra-position similarity increases markedly (0.86 at turn~1, 0.67 at turn~2), indicating that the converged policy produces higher-quality and more consistent tool calls at each position. The right panel further validates our design: the overall intra-position similarity (0.62) exceeds cross-position similarity (0.38) by 63\%, confirming that turn-group normalization compares turns under comparable contexts, whereas pooled normalization conflates quality differences with positional heterogeneity.

\section{Theoretical Analysis}
\label{sec:theory}

This section provides formal justifications for the three components of A\textsuperscript{2}TGPO. Section~\ref{sec:theory-pooled-bias} shows that the pooled normalization adopted by IGPO introduces a systematic, position-dependent bias and that turn-group normalization provably eliminates it. Section~\ref{sec:theory-tg-robust} establishes the unbiasedness and robustness of turn-group normalization, including its behavior under small group sizes arising from variable rollout lengths. Section~\ref{sec:theory-variance} proves that dividing the cumulative IG by square root equalizes advantage variance across turn depths, and Section~\ref{sec:theory-adaclip} characterizes the gradient modulation effect of IG-adaptive clipping.

\subsection{Positional Bias under Pooled Normalization}
\label{sec:theory-pooled-bias}

\begin{assumption}[Heterogeneous per-position IG distributions]
\label{assump:hetero}
Consider a prompt $q$ with $G$ rollouts, where rollout $i$ contains $T_i$ process turns. For each turn position $t$, define the set of rollouts that reach position $t$ as $\mathcal{I}_t = \{i : T_i \geq t\}$ with group size $G_t = |\mathcal{I}_t|$. For each position $t$ with $G_t \geq 2$, the raw information gains $\{\mathrm{ig}_{i,t}\}_{i \in \mathcal{I}_t}$ are drawn from a position-specific distribution $\mathcal{F}_t$ with mean $\mu_t$ and variance $\sigma_t^2 > 0$. The per-position means are heterogeneous, i.e., there exist $t_1 \neq t_2$ such that $\mu_{t_1} \neq \mu_{t_2}$.
\end{assumption}

This assumption is empirically grounded: different turn positions face different contextual states (e.g., the amount of accumulated evidence, the specificity of remaining sub-questions), leading to position-specific IG distributions with distinct means. Since rollouts terminate at different turns, $G_t$ is non-increasing in $t$, with $G_1 = G$ and $G_t$ declining at deeper positions. The analysis below considers positions where $G_t \geq 2$; for the rare positions where $G_t = 1$, the process signal is set to zero in implementation (Section~\ref{sec:method-tgnorm}).

\paragraph{Pooled normalization (IGPO).} IGPO collects all information gains across all positions and all rollouts into a single pool and applies group-wise $z$-normalization. Let $T_{\max} = \max_i T_i$ and let $\mathcal{P} = \{(i, t) : 1 \leq t < T_i\}$ denote the set of all valid (rollout, position) pairs. Define the pooled statistics:
\begin{equation}
  \mu_{\mathrm{pool}} = \frac{1}{|\mathcal{P}|}\sum_{(i,t) \in \mathcal{P}} \mathrm{ig}_{i,t}, \qquad
  \sigma_{\mathrm{pool}}^2 = \frac{1}{|\mathcal{P}|}\sum_{(i,t) \in \mathcal{P}} (\mathrm{ig}_{i,t} - \mu_{\mathrm{pool}})^2.
  \label{eq:pooled-stats}
\end{equation}
The pooled-normalized signal for turn $t$ in rollout $i$ is then
\begin{equation}
  \bar{\mathrm{ig}}_{i,t}^{\,\mathrm{pool}} = \frac{\mathrm{ig}_{i,t} - \mu_{\mathrm{pool}}}{\sigma_{\mathrm{pool}}}.
  \label{eq:pooled-norm}
\end{equation}

\begin{proposition}[Positional bias of pooled normalization]
\label{prop:pooled-bias}
Under Assumption~\ref{assump:hetero}, denote the weighted grand mean $\bar{\mu} = \frac{1}{|\mathcal{P}|}\sum_{t} G_t \mu_t$, where the sum ranges over positions with $G_t \geq 2$. In the population limit ($G_t \to \infty$ for all active positions), the expected value of the pooled-normalized signal at position $t$ satisfies:
\begin{equation}
  \mathbb{E}\!\left[\bar{\mathrm{ig}}_{i,t}^{\,\mathrm{pool}}\right] = \frac{\mu_t - \bar{\mu}}{\sigma_{\mathrm{pool}}},
  \label{eq:pooled-bias}
\end{equation}
where $\sigma_{\mathrm{pool}}^2 = \frac{1}{|\mathcal{P}|}\sum_{t} G_t \sigma_t^2 + \frac{1}{|\mathcal{P}|}\sum_{t} G_t (\mu_t - \bar{\mu})^2$ is the population pooled variance. In particular, this expectation is non-zero whenever $\mu_t \neq \bar{\mu}$, introducing a systematic positional bias that is independent of the quality of the individual tool-call at position $t$.
\end{proposition}

\begin{proof}
Taking the expectation of Eq.~\eqref{eq:pooled-norm} with respect to $\mathrm{ig}_{i,t} \sim \mathcal{F}_t$:
\begin{equation}
  \mathbb{E}\!\left[\bar{\mathrm{ig}}_{i,t}^{\,\mathrm{pool}}\right]
  = \frac{\mathbb{E}[\mathrm{ig}_{i,t}] - \mu_{\mathrm{pool}}}{\sigma_{\mathrm{pool}}}
  = \frac{\mu_t - \bar{\mu}}{\sigma_{\mathrm{pool}}},
\end{equation}
where in the population limit the sample pooled mean converges to the weighted grand mean $\mu_{\mathrm{pool}} \to \bar{\mu} = \frac{1}{|\mathcal{P}|}\sum_{t} G_t \mu_t$, reflecting that positions reached by more rollouts contribute proportionally more to the pool. For the pooled variance, applying $\mathbb{E}[(X - c)^2] = \mathrm{Var}(X) + (\mathbb{E}[X] - c)^2$ to each active position:
\begin{align}
  \sigma_{\mathrm{pool}}^2
  &= \frac{1}{|\mathcal{P}|}\sum_{t} G_t \,\mathbb{E}\!\left[(\mathrm{ig}_{i,t} - \bar{\mu})^2\right] \notag \\
  &= \frac{1}{|\mathcal{P}|}\sum_{t} G_t \left[\sigma_t^2 + (\mu_t - \bar{\mu})^2\right] \notag \\
  &= \underbrace{\frac{1}{|\mathcal{P}|}\sum_{t} G_t \sigma_t^2}_{\text{within-position variance}} \;+\; \underbrace{\frac{1}{|\mathcal{P}|}\sum_{t} G_t(\mu_t - \bar{\mu})^2}_{\text{between-position variance}}.
  \label{eq:pooled-var-decomp}
\end{align}
The between-position term is strictly positive under Assumption~\ref{assump:hetero}. Note that the weighting by $G_t$ implies that early positions (with larger $G_t$) dominate $\bar{\mu}$, further biasing the normalization against the sparse deep positions where $G_t$ is small and $\mu_t$ is likely to deviate from $\bar{\mu}$.
\end{proof}

\paragraph{Conclusion.} Proposition~\ref{prop:pooled-bias} implies that under pooled normalization, turns at positions whose mean IG $\mu_t$ deviates from the weighted grand mean $\bar{\mu}$ receive systematically biased advantages, inflated when $\mu_t > \bar{\mu}$ and deflated when $\mu_t < \bar{\mu}$, regardless of whether the individual action at that position was effective. This positional artifact distorts the advantage signal and causes the optimizer to conflate positional characteristics with action quality.

\subsection{Unbiasedness and Robustness of Turn-Group Normalization}
\label{sec:theory-tg-robust}

In contrast to pooled normalization, A\textsuperscript{2}TGPO normalizes within each turn-group $\mathcal{G}_{q,t} = \{\mathrm{ig}_{i,t}\}_{i \in \mathcal{I}_t}$ separately:
\begin{equation}
  \widehat{\mathrm{ig}}_{i,t}^{\,\mathrm{TG}} = \frac{\mathrm{ig}_{i,t} - \hat{\mu}_t}{\hat{\sigma}_t}, \qquad
  \hat{\mu}_t = \frac{1}{G_t}\sum_{i \in \mathcal{I}_t}\mathrm{ig}_{i,t}, \quad
  \hat{\sigma}_t^2 = \frac{1}{G_t}\sum_{i \in \mathcal{I}_t}(\mathrm{ig}_{i,t} - \hat{\mu}_t)^2.
  \label{eq:tg-norm-theory}
\end{equation}

\begin{corollary}[Unbiasedness of turn-group normalization]
\label{cor:tg-unbiased}
Under Assumption~\ref{assump:hetero}, for any position $t$ with $G_t \geq 2$, in the population limit ($G_t \to \infty$), turn-group normalization yields a zero-mean, unit-variance signal:
\begin{equation}
  \mathbb{E}\!\left[\widehat{\mathrm{ig}}_{i,t}^{\,\mathrm{TG}}\right] = 0, \qquad \mathrm{Var}\!\left(\widehat{\mathrm{ig}}_{i,t}^{\,\mathrm{TG}}\right) = 1, \qquad \forall\, t \text{ with } G_t \geq 2.
  \label{eq:tg-unbiased}
\end{equation}
For the degenerate case $G_t = 1$, we define $\widehat{\mathrm{ig}}_{i,t}^{\,\mathrm{TG}} = 0$, removing the process signal for that turn.
\end{corollary}

\begin{proof}
For any position $t$ with $G_t \geq 2$, the sample mean $\hat{\mu}_t$ converges to $\mu_t$ and the sample standard deviation $\hat{\sigma}_t$ converges to $\sigma_t$ as $G_t \to \infty$. Therefore:
\begin{equation}
  \mathbb{E}\!\left[\widehat{\mathrm{ig}}_{i,t}^{\,\mathrm{TG}}\right]
  = \frac{\mathbb{E}[\mathrm{ig}_{i,t}] - \mu_t}{\sigma_t}
  = \frac{\mu_t - \mu_t}{\sigma_t} = 0.
\end{equation}
For the variance: $\mathrm{Var}\!\left(\frac{\mathrm{ig}_{i,t} - \mu_t}{\sigma_t}\right) = \frac{\sigma_t^2}{\sigma_t^2} = 1$. The $G_t = 1$ convention is consistent with the absence of a meaningful comparison group: with no peer rollout at the same position, no relative quality assessment is possible.
\end{proof}

\paragraph{Interpretation.} The combination of Proposition~\ref{prop:pooled-bias} and Corollary~\ref{cor:tg-unbiased} establishes that turn-group normalization eliminates the positional bias introduced by pooled normalization while placing all positions on a common unit-variance scale. Intuitively, pooled normalization sets its ``zero point'' at the global weighted mean $\bar{\mu}$, so an action of average quality at position $t$ (i.e., $\mathrm{ig}_{i,t} = \mu_t$) receives a non-zero score $(\mu_t - \bar{\mu})/\sigma_{\mathrm{pool}}$ whenever $\mu_t \neq \bar{\mu}$. Turn-group normalization instead anchors the zero point at each position's own mean $\mu_t$, ensuring that an average action at any position scores exactly zero. Proposition~\ref{prop:pooled-bias} further reveals that the bias is exacerbated under variable rollout lengths: the $G_t$-weighted grand mean $\bar{\mu}$ is dominated by positions with larger $G_t$, systematically biasing the normalization at sparse deep positions where $G_t$ is small and $\mu_t$ is likely to deviate from $\bar{\mu}$. Turn-group normalization sidesteps this issue entirely by never mixing statistics across positions. This ensures that the downstream advantage estimator (Eq.~\eqref{eq:atgpo-adv}) reflects the \emph{relative quality} of each action within its positional cohort, rather than the \emph{inherent characteristics} of the position itself. Empirically, this debiasing manifests as the elimination of per-turn mean drift ($\Delta\bar{A}_t = 0.00$ under turn-group normalization vs.\ $\Delta\bar{A}_t = 1.03$ under pooled normalization; cf.\ Figure~\ref{fig:pertu-adv-reconstruct}).

\paragraph{Behavior under small group sizes.}
Since $G_t$ is non-increasing in $t$, deeper turn positions naturally have smaller group sizes for normalization. We analyze the behavior of A\textsuperscript{2}TGPO under this regime and show that the method remains well-behaved.

\textit{Unbiasedness holds for all $G_t \geq 2$.} The zero-mean property in Corollary~\ref{cor:tg-unbiased} does not depend on the magnitude of $G_t$: for any $G_t \geq 2$, the turn-group normalized signal $\widehat{\mathrm{ig}}_{i,t}^{\,\mathrm{TG}}$ is an unbiased estimator with no systematic directional error. A smaller $G_t$ increases the variance of the estimator but does not introduce any positional bias. This stands in contrast to pooled normalization, which introduces a deterministic, position-dependent bias at \emph{every} position, including the shallow ones where $G_t$ is large as established in Proposition~\ref{prop:pooled-bias}. Trading a localized variance increase at a few deep positions for the elimination of a global bias across all positions is a favorable exchange.

\textit{Small $G_t$ coincides with strong outcome signal.} A small $G_t$ at position $t$ means that the majority of the $G$ rollouts terminated before reaching turn $t$. In agentic settings, this typically implies that the few rollouts reaching deeper positions exhibit substantially different behavioral patterns from the early-terminating majority, leading to a pronounced divergence in outcome rewards. In the A\textsuperscript{2}TGPO advantage (Eq.~\eqref{eq:atgpo-adv}), the outcome term $\widehat{R}_i$ is estimated from all $G$ rollouts and is independent of $G_t$. When the outcome divergence is large, $|\widehat{R}_i|$ dominates the advantage, and the process credit term $D_{i,t}/\sqrt{n_{i,t}}$ acts as an additive refinement whose variance is marginal relative to the outcome anchor. Consequently, positions where the normalization statistics are least reliable are precisely those where the advantage is least sensitive to them.

\textit{Bounded influence on clipping.} Even if $\widehat{\mathrm{ig}}_{i,t}$ takes extreme values due to small $G_t$, the sigmoid gating in Eq.~\eqref{eq:clip-scale} confines the adaptive clip scale to the interval $(1-\beta,\;1+\beta)$, providing a hard bound on the influence of any single noisy IG estimate on the policy update.

\subsection{Variance Homogeneity under Square Root Rescaling}
\label{sec:theory-variance}

The backward cumulative information gain $D_{i,t} = \sum_{k=t}^{T_i-1}\gamma^{k-t}\,\widehat{\mathrm{ig}}_{i,k}$ (Eq.~\eqref{eq:cumulative-ig}) sums a variable number of terms $n_{i,t} = T_i - t$ depending on the turn position $t$. Without rescaling, advantages at shallow positions aggregate over longer horizons and carry systematically larger magnitudes than those at deep positions. This subsection shows that dividing by $\sqrt{n_{i,t}}$ is the appropriate correction to equalize advantage variance across turn depths.

\begin{assumption}[Weak dependence after turn-group normalization]
\label{assump:weak-dep}
For a trajectory $\tau_i$ with $T_i$ process turns, the turn-group normalized signals $\{\widehat{\mathrm{ig}}_{i,k}\}_{k=1}^{T_i-1}$ satisfy: (i) $\mathbb{E}[\widehat{\mathrm{ig}}_{i,k}] = 0$ and $\mathrm{Var}(\widehat{\mathrm{ig}}_{i,k}) = \sigma^2$ for all $k$ (following from Corollary~\ref{cor:tg-unbiased} with $\sigma^2 = 1$ in the population limit); and (ii) the pairwise covariances are bounded: $|\mathrm{Cov}(\widehat{\mathrm{ig}}_{i,k},\, \widehat{\mathrm{ig}}_{i,k'})| \leq \rho\,\sigma^2$ for all $k \neq k'$, where $\rho \in [0, 1)$ is a correlation bound.
\end{assumption}

Condition~(ii) reflects the residual dependence within a single trajectory: although turn-group normalization removes the position-level mean shift (Corollary~\ref{cor:tg-unbiased}), turns sharing the same trajectory prefix retain a second-order correlation through their common context. The bound $\rho$ captures the strength of this residual coupling.

\begin{proposition}[Variance equalization under $\sqrt{n_t}$ rescaling]
\label{prop:var-homo}
Under Assumption~\ref{assump:weak-dep} and with $\gamma = 1$, the variance of the backward cumulative information gain at position $t$ is
\begin{equation}
  \mathrm{Var}(D_{i,t}) \;=\; n_{i,t}\,\sigma^2 \;+\; 2\!\binom{n_{i,t}}{2}\,\delta, \qquad |\delta| \leq \rho\,\sigma^2,
  \label{eq:var-dit}
\end{equation}
where $\delta$ denotes the average pairwise covariance among the $n_{i,t}$ terms. After rescaling by $\sqrt{n_{i,t}}$:
\begin{equation}
  \mathrm{Var}\!\left(\frac{D_{i,t}}{\sqrt{n_{i,t}}}\right) \;=\; \sigma^2 \;+\; (n_{i,t} - 1)\,\delta.
  \label{eq:var-rescaled}
\end{equation}
When $\rho$ is small, $(n_{i,t} - 1)\delta \approx 0$ and the rescaled variance is approximately $\sigma^2$ for all $t$, independent of the number of accumulated terms.
\end{proposition}

\begin{proof}
With $\gamma = 1$, $D_{i,t} = \sum_{k=t}^{T_i - 1}\widehat{\mathrm{ig}}_{i,k}$ is a sum of $n_{i,t}$ random variables. Expanding:
\begin{align}
  \mathrm{Var}(D_{i,t})
  &= \sum_{k=t}^{T_i-1}\mathrm{Var}(\widehat{\mathrm{ig}}_{i,k}) \;+\; 2\!\!\sum_{t \leq k < k' \leq T_i-1}\!\!\mathrm{Cov}(\widehat{\mathrm{ig}}_{i,k},\,\widehat{\mathrm{ig}}_{i,k'}) \notag \\
  &= n_{i,t}\,\sigma^2 \;+\; 2\!\binom{n_{i,t}}{2}\,\delta,
  \label{eq:var-expansion}
\end{align}
where $\delta = \binom{n_{i,t}}{2}^{-1}\sum_{k<k'}\mathrm{Cov}(\widehat{\mathrm{ig}}_{i,k},\,\widehat{\mathrm{ig}}_{i,k'})$ is the average pairwise covariance, satisfying $|\delta| \leq \rho\,\sigma^2$ by Assumption~\ref{assump:weak-dep}. Dividing both sides by $n_{i,t}$:
\begin{equation}
  \mathrm{Var}\!\left(\frac{D_{i,t}}{\sqrt{n_{i,t}}}\right)
  = \sigma^2 + \frac{2\binom{n_{i,t}}{2}}{n_{i,t}}\,\delta
  = \sigma^2 + (n_{i,t} - 1)\,\delta.
\end{equation}
When $\rho \to 0$, $\delta \to 0$ and the rescaled variance converges to $\sigma^2$ uniformly across all turn positions.
\end{proof}

\paragraph{Interpretation.} The rescaled advantage $D_{i,t}/\sqrt{n_{i,t}}$ has approximately constant variance $\sigma^2$ across all turn depths, ensuring that no position receives disproportionately large or small gradient signals solely due to the number of accumulated terms. Without the $\sqrt{n_{i,t}}$ divisor, $\mathrm{Var}(D_{i,t})$ would grow linearly in $n_{i,t}$ (or quadratically when correlations are non-negligible), causing shallow turns to dominate the policy update. Dividing by $n_{i,t}$ instead would over-correct, compressing the cumulative signal at shallow positions and erasing the informational advantage of long-horizon credit propagation. The $\sqrt{n_{i,t}}$ rescaling is the standard CLT-motivated choice that balances these extremes.


\subsection{Gradient Modulation under Adaptive Clipping}
\label{sec:theory-adaclip}

The IG-adaptive clip scale $c_{i,t} = 1 + \beta\,(2\sigma(\widehat{\mathrm{ig}}_{i,t}) - 1)$ (Eq.~\eqref{eq:clip-scale}) modulates the effective clipping bounds on a per-turn basis. Since $\sigma(\cdot)$ is strictly monotonically increasing and maps $\mathbb{R} \to (0, 1)$, the scale $c_{i,t}$ is a strictly increasing function of $\widehat{\mathrm{ig}}_{i,t}$, bounded in $(1-\beta,\; 1+\beta)$.

In the clipped policy objective (Eq.~\eqref{eq:atgpo-loss}), when the turn-level ratio $s_{i,t}(\theta)$ exceeds the clipping boundary, the effective gradient contribution of turn $(i, t)$ is constrained by the clip width. For a turn with positive advantage $\widehat{A}_{i,t} > 0$, the upper clipping bound $1 + c_{i,t}\,\epsilon_{\mathrm{high}}$ determines the maximum ratio; for a turn with negative advantage, the lower bound $1 - c_{i,t}\,\epsilon_{\mathrm{low}}$ applies. In both cases, the magnitude of the clipped gradient is proportional to $c_{i,t}\,\epsilon\,|\widehat{A}_{i,t}|$. Since $c_{i,t}$ increases with $\widehat{\mathrm{ig}}_{i,t}$:

\begin{itemize}[nosep, leftmargin=1.5em]
  \item Turns with high information gain ($\widehat{\mathrm{ig}}_{i,t} \gg 0$) receive $c_{i,t} \to 1 + \beta$, widening the trust region and permitting larger policy updates toward actions that demonstrably improved the policy's belief about the correct answer.
  \item Turns with low or negative information gain ($\widehat{\mathrm{ig}}_{i,t} \ll 0$) receive $c_{i,t} \to 1 - \beta$, narrowing the trust region and suppressing updates from turns where tool use provided little or adverse evidence.
\end{itemize}

This selective modulation implements a form of \emph{per-turn trust}: the optimizer allocates a larger step budget to turns whose process signal is informative and a smaller budget to turns where the signal is unreliable, all within the hard bounds guaranteed by the sigmoid saturation. Importantly, the clip scale $c_{i,t}$ is derived from the \emph{normalized} IG $\widehat{\mathrm{ig}}_{i,t}$ (a zero-mean, unit-variance relative ranking within the turn group) rather than from the advantage $\widehat{A}_{i,t}$ itself. This decouples the clipping range modulation from the gradient magnitude: $c_{i,t}$ reflects whether the turn was relatively informative among its positional peers, not whether the overall advantage is large. Consequently, there is no compounding effect between a large advantage and a wide clipping range.

\section{Case Study}
Table~\ref{casestudy1} to Table~\ref{casestudy7} present qualitative examples drawn from all seven evaluation benchmarks. Each table displays a single trajectory produced by A\textsuperscript{2}TGPO during validation, where distinct color annotations mark the model reasoning, search queries, retrieved passages and final answers to visualize the full agentic workflow. A clear behavioral contrast emerges between the two task categories: on multi-hop benchmarks, the agent decomposes the question into sequential sub-goals and launches multiple turns of tool call before synthesizing the answer, whereas on single-hop benchmarks a single well-targeted query typically suffices to locate the required evidence.


\begin{table}[!t]
    \centering
    \caption{An example from A\textsuperscript{2}TGPO on HotpotQA dataset, with special symbols used in think content, search queries, returned results and final answer highlighted with \blue{blue box}, \red{red box}, \green{green box} and \purple{purple box}, respectively.}
    \fontsize{9pt}{11pt}\selectfont
    \begin{tabular}{@{}p{\textwidth}@{}}
    \midrule
        \rowcolor{gray!20}\textbf{Example \#1 from HotpotQA Dataset} \\
    \midrule
        \textbf{Question:} \\
        Which Italian classical composer, conductor, and teacher born in 1750 composed this four-movement serenade in B-flat major for five instruments (2 oboes, 2 horns and 1 bassoon)?\\
        \textbf{Labeled Answer:} \\
        Antonio Salieri \\
    \midrule
        \rowcolor{gray!20}
        \textbf{Model Output by A\textsuperscript{2}TGPO(Qwen3-8B)}\\
    \midrule
        \blue{<think>} Okay, let's try to figure out which Italian classical composer, conductor, and teacher born in 1750 composed the four-movement serenade in B-flat major for five instruments. First, I need to confirm the composer's details. \blue{</think>}
        \\\red{<search>} Italian classical composer, conductor, and teacher born in 1750 \red{</search>}
        \\\green{<result>}\\
Page 1: "Antonio Salieri" (; 18 August 1750 -- 7 May 1825) was an Italian classical composer, conductor, and teacher. He was born in Legnago, south of Verona, in the Republic of Venice, and spent his adult life and career as a subject of the Habsburg Monarchy. Salieri was a pivotal figure in the development of late 18th-century opera. As a student of Florian Leopold Gassmann, and a protégé of Christoph Willibald Gluck, Salieri was a cosmopolitan composer who wrote operas in three languages. Salieri helped to develop and shape many of the features of operatic compositional vocabulary.\\
Page 2: "Antonio Salieri" ...Salieri's life, and especially his relationship with Mozart, has been a subject of many stories, in a variety of media. Antonio Salieri (; 18 August 1750 -- 7 May 1825) was an Italian classical composer, conductor, and teacher. He was born in Legnago, south of Verona, in the Republic of Venice, and spent his adult life and career as a subject.\\
Page 3: "Giacomo Facco" (4 February 1676 -- 16 February 1753) was an Italian Baroque violinist, conductor and composer. One of the most famous Italian composers of his day, he was completely forgotten until 1962, when his work was discovered by composer, conductor, and musicologist Uberto Zanolli.\textbf{***Remaining search results omitted here***}
        \\\green{</result>}
        \\\blue{<think>} Okay, so the composer is Antonio Salieri. Now, the next step is to check what serenade in B-flat major for five instruments he composed. Let me try another search. \blue{</think>}
        \\\red{<search>} Antonio Salieri four-movement serenade in B-flat major for five instruments \red{</search>}
        \\\green{<result>}\\
Page 1: "Picciola serenata" ...Salieri resurged in popularity, namely due to the depiction of Salieri in Peter Shaffer's 1979 play `Amadeus', which was later adapted into the film by the same name. `Picciola Serenata in B-flat major' is scored for 2 oboes, 2 horns and 1 bassoon. The serenade is in four movements as follows: The opening movement begins with a stately entrance in B flat major in which the two oboes present the motive while the bassoon and two horns provide a strong foundation of the tonic chord on the downbeat.\\
Page 2: "Picciola serenata" Picciola serenata (little serenade) in B-flat major is a serenade written by Italian composer Antonio Salieri scored for five instruments: 2 oboes, 2 horns and 1 bassoon. The piece was composed in 1778 and consists of four movements. \textbf{***Remaining search results omitted here***}
        \\\green{</result>}
        \\\blue{<think>} The search result confirms that the Picciola Serenata in B-flat major is by Antonio Salieri, scored for 2 oboes, 2 horns and 1 bassoon, consisting of four movements. Putting it all together, the composer is Antonio Salieri. \blue{</think>}
        \\\purple{<answer>} \boxed{\text{Antonio Salieri}} \purple{</answer>}\\
    \bottomrule

    \label{casestudy1}
    \end{tabular}
\end{table}

\begin{table}[!t]
    \centering
    \caption{An example from A\textsuperscript{2}TGPO on 2WikiMultihopQA dataset, with special symbols used in think content, search queries, returned results and final answer highlighted with \blue{blue box}, \red{red box}, \green{green box} and \purple{purple box}, respectively.}
    \fontsize{9pt}{11pt}\selectfont
    \begin{tabular}{@{}p{\textwidth}@{}}
    \midrule
        \rowcolor{gray!20}\textbf{Example \#2 from 2WikiMultihopQA Dataset} \\
    \midrule
        \textbf{Question:} \\
        What nationality is the director of film Star Trek V: The Final Frontier?\\
        \textbf{Labeled Answer:} \\
        Canadian \\
    \midrule
        \rowcolor{gray!20}
        \textbf{Model Output by A\textsuperscript{2}TGPO(Qwen3-8B)}\\
    \midrule
        \blue{<think>} Okay, let's try to figure out what nationality the director of Star Trek V: The Final Frontier is. First, I need to confirm who the director of that film is. \blue{</think>}
        \\\red{<search>} Who directed Star Trek V: The Final Frontier? \red{</search>}
        \\\green{<result>}\\
Page 1: "Star Trek V: The Final Frontier" Star Trek V: The Final Frontier is a 1989 American science fiction film directed by William Shatner and based on the created by Gene Roddenberry. It is the fifth installment in the "Star Trek" film series. Taking place shortly after the events of (1986), its plot follows the crew of the USS "Enterprise"-A as they confront a renegade Vulcan, Sybok, who is searching for God at the center of the galaxy. The film was directed by cast member William Shatner, following two films being directed by his co-star Leonard Nimoy. Shatner also developed...\\
Page 2: "Star Trek V: The Final Frontier" ...Star Trek V: The Final Frontier is a 1989 American science fiction film directed by William Shatner and based on the created by Gene Roddenberry. It is the fifth installment in the "Star Trek" film series...\textbf{***Remaining search results omitted here***}
        \\\green{</result>}
        \\\blue{<think>} Okay, so the director is William Shatner. Now, I need to check his nationality. Let me look that up. \blue{</think>}
        \\\red{<search>} William Shatner's nationality \red{</search>}
        \\\green{<result>}\\
Page 1: "William Shatner" ...Shatner has also worked as a musician; an author; screenwriter and director; celebrity pitchman; and a passionate owner, trader, breeder, rider, and aficionado of horses. Shatner was born in the Notre-Dame-de-Grâce neighbourhood of Montréal, Québec, Canada, to a Conservative Jewish household. His parents are Ann (née Garmaise) and Joseph Shatner, a clothing manufacturer. \\
Page 2: "William Shatner" William Shatner, (born March 22, 1931) is a Canadian actor, author, producer, director and singer. In his seven decades of television, Shatner became a cultural icon for his portrayal of James T. Kirk, captain of the USS "Enterprise", in the "Star Trek" franchise. He has written a series of books chronicling his experiences playing Captain Kirk and being a part of "Star Trek", and has co-written several novels set in the "Star Trek" universe.\\
Page 3: "William Shatner" ...Through PETA, Shatner asked that the 2018 cruise not offer any "swim with dolphins" experiences, explaining in a letter to the Norwegian Cruise Line's CEO that "The exploitation of any species for profit and entertainment would have violated the Prime Directive." Shatner has been in over 20 films and 30 television shows along with video games...\textbf{***Remaining search results omitted here***}
        \\\green{</result>}
        \\\blue{<think>} According to the search results, William Shatner was born in Montréal, Québec, Canada and is identified as a Canadian actor. The director of Star Trek V: The Final Frontier is William Shatner, who is Canadian. \blue{</think>}
        \\\purple{<answer>} \boxed{\text{Canadian}} \purple{</answer>}\\
    \bottomrule

    \label{casestudy2}
    \end{tabular}
\end{table}

\begin{table}[!t]
    \centering
    \caption{An example from A\textsuperscript{2}TGPO on MuSiQue dataset, with special symbols used in think content, search queries, returned results and final answer highlighted with \blue{blue box}, \red{red box}, \green{green box} and \purple{purple box}, respectively.}
    \fontsize{9pt}{11pt}\selectfont
    \begin{tabular}{@{}p{\textwidth}@{}}
    \midrule
        \rowcolor{gray!20}\textbf{Example \#3 from MuSiQue Dataset} \\
    \midrule
        \textbf{Question:} \\
        Who is the spouse of the actor that plays Vision in Avengers: Age of Ultron?\\
        \textbf{Labeled Answer:} \\
        Jennifer Connelly \\
    \midrule
        \rowcolor{gray!20}
        \textbf{Model Output by A\textsuperscript{2}TGPO(Qwen3-8B)}\\
    \midrule
        \blue{<think>} Okay, let's try to figure out who the spouse of the actor who plays Vision in Avengers: Age of Ultron is. First, I need to confirm who the actor is. \blue{</think>}
        \\\red{<search>} Actor who plays Vision in Avengers: Age of Ultron \red{</search>}
        \\\green{<result>}\\
Page 1: "Paul Bettany" Paul Bettany (born 27 May 1971) is an English actor. He is known for his voice role as J.A.R.V.I.S. and as the Vision in the Marvel Cinematic Universe, specifically the films "Iron Man" (2008), "Iron Man 2" (2010), "The Avengers" (2012), "Iron Man 3" (2013), (2015), (2016), (2018), and (2019). He first came to the attention of mainstream audiences when he appeared in the British film "Gangster No. 1" (2000), and director Brian Helgeland's film "A Knight's Tale" (2001).\\
Page 2: "Vision (Marvel Comics)" The Vision is a fictional superhero appearing in American comic books published by Marvel Comics. He is an android and a member of the Avengers who first appeared in "The Avengers" \#57 (October 1968). The character is portrayed by Paul Bettany in the Marvel Cinematic Universe in the 2015 film, the 2016 film, and the 2018 film...\textbf{***Remaining search results omitted here***}
        \\\green{</result>}
        \\\blue{<think>} Okay, so the actor is Paul Bettany. Now, the next step is to find out who his spouse is. Let me check that. \blue{</think>}
        \\\red{<search>} Paul Bettany's spouse \red{</search>}
        \\\green{<result>}\\
Page 1: "Paul Bettany" ...He has been nominated for various awards, including BAFTA Award for Best Actor in a Supporting Role and a Screen Actors Guild Award for Outstanding Performance by a Cast in a Motion Picture. Bettany is married to American actress Jennifer Connelly, with whom he has two children. His most commercially successful films have been "Avengers: Infinity War," which has grossed over US\$2 billion, "The Avengers," which grossed over US\$1.5 billion...\\
Page 2: "Paul Bettany" On 1 January 2003, Bettany married American actress Jennifer Connelly in Scotland; they met when they starred together in "A Beautiful Mind." They did not become a couple until some time after the film, as both were in other relationships during filming. She was his childhood crush since he first saw her as a teenager in the film "Labyrinth".\\
Page 3: "Jennifer Connelly" On January 1, 2003, in a private family ceremony in Scotland she married actor Paul Bettany, whom she had met while working on "A Beautiful Mind". The couple's first child Stellan was born on August 5 the same year...\textbf{***Remaining search results omitted here***}
        \\\green{</result>}
        \\\blue{<think>} According to the search results, Paul Bettany is married to American actress Jennifer Connelly. They met while starring together in "A Beautiful Mind" and married on January 1, 2003. Therefore, the spouse of the actor who plays Vision is Jennifer Connelly. \blue{</think>}
        \\\purple{<answer>} \boxed{\text{Jennifer Connelly}} \purple{</answer>}\\
    \bottomrule

    \label{casestudy3}
    \end{tabular}
\end{table}

\begin{table}[!t]
    \centering
    \caption{An example from A\textsuperscript{2}TGPO on Bamboogle dataset, with special symbols used in think content, search queries, returned results and final answer highlighted with \blue{blue box}, \red{red box}, \green{green box} and \purple{purple box}, respectively.}
    \fontsize{9pt}{11pt}\selectfont
    \begin{tabular}{@{}p{\textwidth}@{}}
    \midrule
        \rowcolor{gray!20}\textbf{Example \#4 from Bamboogle Dataset} \\
    \midrule
        \textbf{Question:} \\
        What's the capital of the state that the Space Needle is at?\\
        \textbf{Labeled Answer:} \\
        Olympia \\
    \midrule
        \rowcolor{gray!20}
        \textbf{Model Output by A\textsuperscript{2}TGPO(Qwen3-8B)}\\
    \midrule
        \blue{<think>} Okay, let's try to figure out what the capital of the state where the Space Needle is located. First, I need to confirm where the Space Needle is. \blue{</think>}
        \\\red{<search>} Where is the Space Needle located? \red{</search>}
        \\\green{<result>}\\
Page 1: "Space Needle" The Space Needle is an observation tower in Seattle, Washington, a landmark of the Pacific Northwest and an icon of Seattle. It was built in the Seattle Center for the 1962 World's Fair, which drew over 2.3 million visitors. Nearly 20,000 people a day used its elevators. Once the tallest structure west of the Mississippi River, it is high, wide, and weighs. It is built to withstand winds of up to and earthquakes of up to 9.0 magnitude, as strong as the 1700 Cascadia earthquake.\\
Page 2: "Space Needle" ...The Space Needle is an observation tower in Seattle, Washington, a landmark of the Pacific Northwest and an icon of Seattle. It was built in the Seattle Center for the 1962 World's Fair...\\
Page 3: "Space Needle" ...deck at and the rotating SkyCity restaurant at. The downtown Seattle skyline, as well as the Olympic and Cascade Mountains, Mount Rainier, Mount Baker, Elliott Bay and surrounding islands can be viewed from the top of the Needle...\textbf{***Remaining search results omitted here***}
        \\\green{</result>}
        \\\blue{<think>} Okay, so the Space Needle is in Seattle, Washington. Now, the next step is to find out what the capital of Washington state is. Let me check that. \blue{</think>}
        \\\red{<search>} What is the capital of Washington state? \red{</search>}
        \\\green{<result>}\\
Page 1: "Washington State Capitol" ...The Washington State Capitol or "Legislative Building" in Olympia is the home of the government of the state of Washington. It contains chambers for the Washington State Legislature and offices for the governor, lieutenant governor, secretary of state and treasurer and is part of a campus consisting of several buildings.\\
Page 2: "Washington State Capitol" ...After Olympia became capital city of the Washington Territory in 1853, the city's founder, Edmund Sylvester, gave the legislature of land upon which to build the capitol...\\
Page 3: "Washington (state)" Washington, officially the State of Washington, is a state in the Pacific Northwest region of the United States. Named for George Washington, the first president of the United States, the state was made out of the western part of the Washington Territory, which was ceded by Britain in 1846 in accordance with the Oregon Treaty in the settlement of the Oregon boundary dispute. It was admitted to the Union as the 42nd state in 1889. Olympia is the state capital; the state's largest city is Seattle...\textbf{***Remaining search results omitted here***}
        \\\green{</result>}
        \\\blue{<think>} The search results confirm that Olympia is the capital of Washington state. Putting it all together: the Space Needle is in Seattle, Washington, and the capital of Washington is Olympia. \blue{</think>}
        \\\purple{<answer>} \boxed{\text{Olympia}} \purple{</answer>}\\
    \bottomrule

    \label{casestudy4}
    \end{tabular}
\end{table}

\begin{table}[!t]
    \centering
    \caption{An example from A\textsuperscript{2}TGPO on Natural Questions (NQ) dataset, with special symbols used in think content, search queries, returned results and final answer highlighted with \blue{blue box}, \red{red box}, \green{green box} and \purple{purple box}, respectively.}
    \fontsize{9pt}{11pt}\selectfont
    \begin{tabular}{@{}p{\textwidth}@{}}
    \midrule
        \rowcolor{gray!20}\textbf{Example \#5 from Natural Questions (NQ) Dataset} \\
    \midrule
        \textbf{Question:} \\
        Where is the World Cup being held 2018?\\
        \textbf{Labeled Answer:} \\
        Russia \\
    \midrule
        \rowcolor{gray!20}
        \textbf{Model Output by A\textsuperscript{2}TGPO(Qwen3-8B)}\\
    \midrule
        \blue{<think>} Okay, so I need to find out where the 2018 World Cup was held. First, I should recall what I know. The World Cup is a major international football (soccer) tournament. I think the 2018 World Cup was in Russia. But wait, I should verify this to be sure. Let me check Wikipedia. \blue{</think>}
        \\\red{<search>} When was the 2018 FIFA World Cup held and where was it held? \red{</search>}
        \\\green{<result>}\\
Page 1: "2018 FIFA World Cup" The 2018 FIFA World Cup was the 21st FIFA World Cup, an international football tournament contested by the men's national teams of the member associations of FIFA once every four years. It took place in Russia from 14 June to 15 July 2018. It was the first World Cup to be held in Eastern Europe, and the 11th time that it had been held in Europe. At an estimated cost of over \$14.2 billion, it was the most expensive World Cup. It was also the first World Cup to use the video assistant referee (VAR) system.\\
Page 2: "2018 FIFA World Cup" ...The 2018 FIFA World Cup was the 21st FIFA World Cup, an international football tournament contested by the men's national teams of the member associations of FIFA once every four years. It took place in Russia from 14 June to 15 July 2018. It was the first World Cup to be held in Eastern Europe, and the 11th time that it had been held in Europe. At an estimated cost of over \$14.2 billion, it was the most expensive World Cup...\\
Page 3: "2018 FIFA World Cup Final" The 2018 FIFA World Cup Final was a football match that took place on 15 July 2018 to determine the winners of the 2018 FIFA World Cup. It was the final of the 21st FIFA World Cup, a quadrennial tournament contested by the men's national teams of the member associations of FIFA. The match was contested by France and Croatia, and held at the Luzhniki Stadium in Moscow, Russia. Before 2018, France's only World Cup victory was in 1998, while Croatia were playing in their first ever World Cup final...\textbf{***Remaining search results omitted here***}
        \\\green{</result>}
        \\\blue{<think>} The search results consistently confirm that the 2018 FIFA World Cup was held in Russia from June 14 to July 15, 2018. It was the first World Cup hosted in Eastern Europe, and the 11th held in Europe overall. The final match took place at the Luzhniki Stadium in Moscow, contested by France and Croatia. So the answer is Russia. \blue{</think>}
        \\\purple{<answer>} \boxed{\text{Russia}} \purple{</answer>}\\
    \bottomrule

    \label{casestudy5}
    \end{tabular}
\end{table}

\begin{table}[!t]
    \centering
    \caption{An example from A\textsuperscript{2}TGPO on TriviaQA dataset, with special symbols used in think content, search queries, returned results and final answer highlighted with \blue{blue box}, \red{red box}, \green{green box} and \purple{purple box}, respectively.}
    \fontsize{9pt}{11pt}\selectfont
    \begin{tabular}{@{}p{\textwidth}@{}}
    \midrule
        \rowcolor{gray!20}\textbf{Example \#6 from TriviaQA Dataset} \\
    \midrule
        \textbf{Question:} \\
        Which band released the albums Sheer Heart Attack and A Day at the Races?\\
        \textbf{Labeled Answer:} \\
        Queen \\
    \midrule
        \rowcolor{gray!20}
        \textbf{Model Output by A\textsuperscript{2}TGPO(Qwen3-8B)}\\
    \midrule
        \blue{<think>} Okay, so I need to find out which band released the albums "Sheer Heart Attack" and "A Day at the Races." First, I should recall what I know. I think these are classic rock albums, but I'm a bit fuzzy on the exact band. Maybe I should check Wikipedia for information on these albums. \blue{</think>}
        \\\red{<search>} Who released the albums "Sheer Heart Attack" and "A Day at the Races"? \red{</search>}
        \\\green{<result>}\\
Page 1: "Sheer Heart Attack" Sheer Heart Attack is the third studio album by the British rock band Queen, released on 8 November 1974 by EMI Records in the United Kingdom and by Elektra Records in the United States. Digressing from the progressive themes featured on their first two albums, this album featured more pop centric and conventional rock tracks and marked a step towards the "classic" Queen sound. It was produced by the band and Roy Thomas Baker and launched Queen to mainstream popularity in the UK and throughout the world. After the release of "Queen II", the group embarked on...\\
Page 2: "A Day at the Races (album)" A Day at the Races is the fifth studio album by the British rock band Queen, released on 10 December 1976 by EMI Records in the United Kingdom and by Elektra Records in the United States. It was the band's first completely self-produced album, and the first not to feature producer Roy Thomas Baker. Recorded at Sarm East, The Manor and Wessex Studios in England, "A Day at the Races" was engineered by Mike Stone. The album serves as a companion album to the band's previous album, "A Night at the Opera", both taking...\textbf{***Remaining search results omitted here***}
        \\\green{</result>}
        \\\blue{<think>} The search results clearly confirm that both "Sheer Heart Attack" (1974, third studio album) and "A Day at the Races" (1976, fifth studio album) were released by the British rock band Queen. The first album marked a step towards the classic Queen sound, while the second served as a companion to "A Night at the Opera." So the answer is Queen. \blue{</think>}
        \\\purple{<answer>} \boxed{\text{Queen}} \purple{</answer>}\\
    \bottomrule

    \label{casestudy6}
    \end{tabular}
\end{table}

\begin{table}[!t]
    \centering
    \caption{An example from A\textsuperscript{2}TGPO on PopQA dataset, with special symbols used in think content, search queries, returned results and final answer highlighted with \blue{blue box}, \red{red box}, \green{green box} and \purple{purple box}, respectively.}
    \fontsize{9pt}{11pt}\selectfont
    \begin{tabular}{@{}p{\textwidth}@{}}
    \midrule
        \rowcolor{gray!20}\textbf{Example \#7 from PopQA Dataset} \\
    \midrule
        \textbf{Question:} \\
        Who is the author of The Reader?\\
        \textbf{Labeled Answer:} \\
        Bernhard Schlink \\
    \midrule
        \rowcolor{gray!20}
        \textbf{Model Output by A\textsuperscript{2}TGPO(Qwen3-8B)}\\
    \midrule
        \blue{<think>} Okay, I need to find out who the author of "The Reader" is. First, I should recall what I know. I think "The Reader" is a novel, but I'm a bit fuzzy on the author. Maybe I should check Wikipedia for more information. \blue{</think>}
        \\\red{<search>} Who wrote the novel "The Reader"? \red{</search>}
        \\\green{<result>}\\
Page 1: "The Reader" The Reader ("Der Vorleser") is a novel by German law professor and judge Bernhard Schlink, published in Germany in 1995 and in the United States in 1997. The story is a parable, dealing with the difficulties post-war German generations have had comprehending the Holocaust; Ruth Franklin writes that it was aimed specifically at the generation Bertolt Brecht called the "Nachgeborenen", those who came after. Like other novels in the genre of "Vergangenheitsbewältigung", the struggle to come to terms with the past, "The Reader" explores how the post-war generations should approach the generation that took part in, or witnessed...\\
Page 2: "The Reader (2008 film)" The Reader is a 2008 German-American romantic drama film directed by Stephen Daldry and written by David Hare, based on the 1995 German novel of the same name by Bernhard Schlink. Ralph Fiennes and Kate Winslet star along with the young actor David Kross. It was the last film for producers Anthony Minghella and Sydney Pollack, both of whom died prior to its release. Production began in Germany in September 2007, and the film opened in limited release on December 10, 2008. The film tells the story of Michael Berg, a German lawyer who, as a...\textbf{***Remaining search results omitted here***}
        \\\green{</result>}
        \\\blue{<think>} The search results clearly state that "The Reader" ("Der Vorleser") is a novel by German law professor and judge Bernhard Schlink, published in 1995. The novel deals with the difficulties post-war German generations have had comprehending the Holocaust. The 2008 film adaptation directed by Stephen Daldry, starring Ralph Fiennes and Kate Winslet, also confirms Bernhard Schlink as the author of the original novel. So the answer is Bernhard Schlink. \blue{</think>}
        \\\purple{<answer>} \boxed{\text{Bernhard Schlink}} \purple{</answer>}\\
    \bottomrule

    \label{casestudy7}
    \end{tabular}
\end{table}



\end{document}